\documentclass[twoside,11pt]{article}

\usepackage{jmlr2e}
\usepackage{xcolor}

\usepackage{graphicx}
\usepackage{subfigure}
\usepackage{amssymb}
\usepackage{amsmath}

\usepackage{xcolor}

 \usepackage{algorithm}
 \usepackage{algorithmic}

\jmlrheading{?}{2020}{??}{??}{??}{meila00a}{Yueming Lyu and Yuan Yuan and Ivor W. Tsang}

\ShortHeadings{Efficient Batch Black-box Optimization with Deterministic Regret Bounds}{Yueming Lyu and Yuan Yuan and Ivor W. Tsang}
\firstpageno{1}

\begin{document}

\title{Efficient Batch Black-box Optimization with Deterministic Regret Bounds}

\author{\name Yueming Lyu \email yueminglyu@gmail.com \\
       \addr Centre for Artificial Intelligence\\
       University of Technology Sydney\\
       15 Broadway, Ultimo NSW 2007, Sydney, Australia
       \AND
       \name Yuan Yuan\email miayuan@mit.edu \\
       \addr CSAIL\\
       Massachusetts Institute of Technology\\
       77 Massachusetts Avenue, Cambridge, MA, USA 
       \AND 
       \name Ivor W. Tsang \email Ivor.Tsang@uts.edu.au \\
       \addr Centre for Artificial Intelligence\\
       University of Technology Sydney\\
       15 Broadway, Ultimo NSW 2007, Sydney, Australia}

\editor{--}

\maketitle

\begin{abstract}
 In this work, we investigate black-box optimization from the perspective of frequentist kernel methods. We propose a novel batch optimization algorithm, which jointly maximize the acquisition function and select points from a  whole batch in a holistic way. Theoretically, we derive regret bounds for both the noise-free and perturbation settings irrespective of the choice of kernel. Moreover, we analyze the property of the adversarial regret that is required by a robust initialization for Bayesian Optimization (BO). We prove that the adversarial regret bounds decrease with the decrease of covering radius, which provides a criterion for generating a point set 
 to minimize the bound. We then propose fast searching algorithms to generate a point set with a small covering radius for the robust initialization.  Experimental results on both synthetic benchmark problems and real-world problems show the effectiveness of the proposed algorithms.
\end{abstract}


\section{Introduction}

Bayesian Optimization (BO) is a promising approach to address expensive black-box (non-convex) optimization problems.  Applications of BO include automatic tuning of hyper-parameters in machine learning~\citep{Nips2012practical}, gait optimization in robot control~\citep{robot}, molecular compounds identifying in drug discovery~\citep{drugdiscovery}, and optimization of computation-intensive engineering design~\citep{engineerDisign}.

BO aims to find the optimum of an unknown, usually non-convex function $f$. Since little information is known about the underlying function $f$, BO requires to estimate a surrogate function to model the unknown function.
Therefore, one major challenge of BO is to seek a balance between collecting information to model the function $f$ (exploration) and searching for an optimum based on the collected information (exploitation).

Typically, BO assumes that the underlying function $f$ is sampled from a Gaussian process (GP) prior. 
BO selects the candidate solutions for evaluation by maximizing an acquisition function~\citep{PI,ei1975,EI}), which balances the exploration and exploitation given all previous observations.
In practice, BO can usually find an approximate maximum solution with a remarkably small number of function evaluations~\citep{Nips2012practical,scarlett2018tight}.

In many real applications, parallel processing of multiple function evaluations is usually  preferred to achieve the time efficiency. For example, examining various hyper-parameter settings of a machine learning algorithm simultaneously or running multiple instances of a reinforcement learning simulation in parallel can save time in a large margin.
Sequential BO selection is by no means efficient in these scenarios.
Therefore, several batch BO approaches have been proposed to address this issue. \cite{shah2015parallel} propose a parallel predictive entropy search method, which is based on the predictive entropy search (PES)~\citep{PES} and extends PES to the batch case.
\cite{wu2016parallel} generalize the knowledge gradient method~\citep{frazier2009knowledge} to a parallel knowledge gradient method. However, these methods are still computationally intensive because they all rely on expensive Monte Carlo sampling. 
Moreover, they are not scalable to large batch size and lack the theoretical convergence guarantee as well. 

Instead of using Monte Carlo sampling, another line of research improves the efficiency by deriving tighter upper confidence bounds. The GP-BUCB policy~\citep{GPBUCB} makes the selections point-by-point sequentially until reaching a pre-set batch size, according to upper confidence bound (UCB) criterion~\citep{auer2002using,gpucb} with a fixed mean function and an updated covariance function. It proves a sub-linear growth bound on the cumulative regret, which guarantees the bound on the number of required iterations to get close enough to the optimum. The GP-UCB-PE~\citep{UCBPE} combines the UCB strategy and the pure exploration~\citep{PE} in the evaluations of the same batch, achieving a better upper bound on the cumulative regret compared with the GP-BUCB. Although these methods do not require any Monte Carlo sampling, they select candidate queries of a batch greedily. As a result, they are still far from satisfactory in terms of both efficiency and scalability.

 To achieve greater efficiency in the batch selection, we propose to simultaneously select candidate queries of a batch in a holistic manner, rather than the previous sequential manner. In this paper, we analyze both the batch BO and the sequential BO from a frequentist perspective. For the batch BO, we propose a novel batch selection method that takes both the mean prediction value and the correlation of points in a batch into consideration.   Our method leads to a novel batch acquisition function.
 By jointly maximizing the novel acquisition function with respect to all the points in a batch, the proposed method is able to attain a better exploitation/exploration trade-off. 

 For the sequential BO, we obtain a similar acquisition function as that in the GP-UCB~\citep{gpucb}, except that our function employs a constant weight for the deviation term. The constant weight is preferred over the previous theoretical weight proposed in GP-UCB, because the latter is overly conservative, which has been observed in many other works~\citep{bogunovic2016truncated,bogunovic2018adversarially,gpucb}.
Moreover, for functions with a bounded norm in the reproducing kernel Hilbert space (RKHS), we derive the non-trivial regret bounds for both the batch BO method and the sequential BO method.

At the beginning of the BO process, since little information is known, the initialization phase becomes vitally important.
To obtain a good and robust initialization, we first study the properties which are necessary for a robust initialization through analyzing the adversarial regret. We prove that the regret bounds decrease with the decrease of the covering radius (named fill distance in \citep{kanagawa2018gaussian}). Minimizing the covering radius of a lattice is equivalent to maximizing its packing radius (named separate distance in \citep{kanagawa2018gaussian} )~\citep{rank1Image,keller2007monte}, we then propose a novel fast searching method to maximize the packing radius of a rank-1 lattice and obtain the points set with a small covering radius.





Our contributions are summarized as follows:

\begin{itemize}
    \item We study the black-box optimization for functions with a bounded norm in RKHS and achieve deterministic regret bounds for both the noise-free setting and the perturbation setting. The study not only brings a new insight into the BO literature but also provides better guidance for designing new acquisition functions.

   \item We propose a more-efficient novel adaptive algorithm for batch optimization, which selects candidate queries of a batch in a holistic manner. Theoretically, we prove that the proposed methods achieve non-trivial regret bounds.

    \item We analyze the adversarial regret for a robust initialization of BO, and theoretically prove that the regret bounds decrease with the decrease of the covering radius, and provide a criterion for generating points set to minimize the bound for the initialization of BO.

    \item We propose a novel, fast searching algorithm to maximize the packing radius of a rank-1 lattice and generate a set of points with a small covering radius. The generated points set provides a robust initialization for BO. Moreover, the set of points can be used for integral approximation on domain $[0,1]^d$. Experimental results show that the proposed method can achieve a larger packing radius (separate distance) compared with the baselines.

\end{itemize}

\section{Related Work}

Black-box optimization has been investigated by different communities for several decades. 
In the mathematical optimization community, derivative-free optimization (DFO) methods are widely studied for black-box optimization. These methods can be further divided into three categories:  direct search methods,  model-based methods, and random search methods. Amongst them, the model-based methods guide the searching procedure by using the model prediction as to the surrogate, which is quite similar to the Bayesian optimization methods. We refer to \citep{rios2013derivative}, \citep{audet2017derivative} and \citep{derivative2} for detailed survey of the derivative-free optimization methods. 
In the evolutionary computation community, researchers have developed the evolutionary algorithm~\citep{srinivas1994genetic} and evolutionary strategy methods~\citep{back1991survey} for the black-box optimization, where the latter is similar to the Nesterov random search~\citep{nesterov2017random} in the DFO methods since both the evolutionary strategy methods and the Nesterov random search employ the Gaussian smoothing technique to approximate the gradient. 
In the machine learning community, investigating the black-box optimization from the aspect of Bayesian optimization (BO) has attracted more and more attention recently. BO has been successfully applied to address many expensive black-box optimization problems, such as hyper-parameter tuning for deep networks~\citep{Nips2012practical}, parametric policy optimization for Reinforcement learning~\citep{wilson2014using}, and so on. Since our proposed method belongs to the BO category, we mainly focus on the review and discussion about the BO related works in the following paragraphs of the related work section.


The research of BO for black-box optimization can be dated back to \citep{movckus1975bayesian}. It becomes popular since the efficient global optimization method~\citep{jones1998efficient} for black-box optimization has been proposed. After that, various acquisition functions have been widely investigated both empirically and theoretically.  Acquisition functions are important in BO as they determine the searching behavior. Among them, expected improvement, probability improvement and upper confidence bound of the Gaussian process (GP-UCB) are the most widely used acquisition functions in practice~\citep{Nips2012practical}. Specifically, \cite{bull} has proved a simple regret bound of the expected improvement-based method. \cite{gpucb} have theoretically analyzed both the cumulative regret and the simple regret bounds of the GP-UCB method.   

Recently, many sophisticated acquisition functions have been studied. \cite{hennig2012entropy} propose entropy search (ES) method,  \cite{PES} further propose a predictive entropy search (PES) method.  Both ES and PES  select the candidate query by maximizing the mutual information between the query point and the global optimum in the input space. As a result, they need intensive Monte Carlo sampling that depends on the dimension of the input space. To reduce the cost of sampling, \cite{wang2017max} propose a max-value entropy search method, selecting the candidate query by maximizing the mutual information between the prediction of the query and the maximum value. The mutual information is computed in one dimension, which is much easier to approximate compared to the Monte Carlo sampling. Along the line of GP-UCB, \cite{GPBUCB} propose the GP-BUCB method to address the black-box optimization in a batch setting. In each batch, GP-BUCB selects the candidate queries point-by-point sequentially until reaching a pre-set batch size, according to upper confidence bound criterion~\citep{auer2002using,gpucb} with a fixed mean function and an updated covariance function. \cite{GPBUCB} prove a sub-linear growth bound on the cumulative regret, which guarantees a bound on the number of required iterations to reach sufficiently close to the optimum. \cite{UCBPE} further propose the GP-UCB-PE method, which combines the upper confidence bound strategy and the pure exploration~\citep{PE} in the evaluations of the same batch. The GP-UCB-PE achieves a better upper bound on the cumulative regret compared with the GP-BUCB. Most recently, \cite{berkenkamp2019no} propose a GP-UCB based method (A-GP-UCB) to handle BO with unknown hyper-parameters. 

 In many applications, e.g, hyper-parameter tuning and RL,  it is usually preferred to process multiple function evaluations in parallel to achieve the time efficiency. In the setting of batch BO for batch black-box optimization, besides the batch, GP-UCB methods discussed above, \cite{shah2015parallel} propose a parallel predictive entropy search method by extending the PES method~\citep{PES} to the batch case. \cite{wu2016parallel} extend the knowledge gradient method to the parallel knowledge gradient method. \cite{gonzalez2016batch} propose a penalized acquisition function for batch selection. However, these batch methods are limited in low dimensional problems.  To address the batch BO under the high-dimensional setting,  \cite{wang2017batched} propose an ensemble BO method by integrating multiple additive Gaussian process models.  However, no regret bound is analyzed in \citep{wang2017batched}.
Our work belongs to the GP-UCB family. Different from existing GP-UCB methods above, we study BO in a frequentist perspective, and we prove deterministic bounds for both the sequential and batch settings. A most related method is Bull's method~\citep{bull}. The limitations of Bull's method and the relationships are listed as follows.

\textbf{Limitations of Bull's batch method:} 
 \cite{bull} presents a non-adaptive batch-based method with all the query points except one being fixed at the beginning. However, as mentioned by Bull, this method is not practical. We propose an adaptive BO method and initialize it with a robust initialization algorithm. More specifically, we first select the initialization query points by minimizing the covering radius and then select the query points based on our adaptive methods.

\textbf{Relationship to Bull's bounds:}
We give the regret bound w.r.t. the covering radius for different kernels; while Bull's bound is limited to Mart\'ern type kernel. 
Compared with Bull's bound (Theorem~1 in \citep{bull}), our regret bound directly links to the covering radius (fill distance), which provides a criterion for generating a point set to achieve small bounds. In contrast, Bull's bound does not provide a criterion for minimizing the bound. We generate the initialization point set by minimizing covering radius (our bound); while Bull's work doesn't.

\section{Notations and Symbols}\label{notation}
Let  $\mathcal{H}_k$  denote a separable reproducing kernel Hilbert space associated with the kernel $k( \cdot ,\cdot)$, and Let ${\left\| \cdot \right\|_{{\mathcal{H}_k}}}$ denote the RKHS norm in ${\mathcal{H}_k}$ .   $\|\cdot\|$ denotes the $l_2$ norm (Euclidean distance). Let $\mathcal{B}_k= \{f: f \in \mathcal{H}_k, {\left\| f \right\|_{{\mathcal{H}_k}}} \le B\}$ denotes a bounded subset in the RKHS, and $ \mathcal{X} \subset \mathbb{R}^d$ denote a compact set in $\mathbb{R}^d$. Symbol $[N]$ denotes the set $\{1,...,N\}$. 
$\mathbb{N}$ and $\mathbb{P}$ denote the integer set and prime number set, respectively. Bold capital letters are used for matrices.

\section{Background and Problem Setup}\label{background}

Let  $f:\; \mathcal{X} \to \mathbb{R}$ be the unknown black-box function to be optimized,  where $ \mathcal{X} \subset {\mathbb{R}^d}$ is a compact set. BO aims to find a maximum $x^*$ of the function $f$, i.e., \[f(x^*)={\max _{x \in \mathcal{X} }}f(x).\] 

In sequential BO, a single point ${x_t} \in \mathcal{X}$ is selected to query  an  observation at round $t$. Batch BO is introduced in the literature for the benefits of parallel execution.  Batch BO methods select a batch of points $X_n=\{ {x_{(n-1)L+1},...,x_{nL} }\} $   simultaneously at round $n$, where $L$ is the batch size.
The batch BO is different from the sequential BO because the observation is delayed for batch BO during the batch selection phase. An additional challenge is introduced in batch BO since it needs to select a batch of points at one time,  without knowing the latest information about the function $f$.

In BO, the effectiveness of a selection policy can be measured by the cumulative regret $R_T$ and the simple regret (minimum regret) $r_T$ over $T$ steps. The cumulative regret $R_T$ and simple regret $r_T$ are defined as follows, 
\begin{eqnarray}
 {R_T} &=& \sum\limits_{t = 1}^T \left( {f({x^*}) - f({x_t})} \right),  \label{RT}\\
   {r_T} & = & f({x^*}) - \mathop {\max }\limits_{1 \le t \le T} f({x_t}).  \label{rT}
 \end{eqnarray}
The regret bound introduced in numerous theoretical works is based on the maximum information gain defined as 
\begin{equation}\label{MutualInformation}
    \gamma_T = \max _{\bf{x}_1,...,\bf{x}_T} \frac{1}{2}\log \det ({\bf{I}}_T + {\sigma ^{ - 2}}{{\bf{K}}_T}).
\end{equation}
The bounds of $\gamma_T$ for commonly used kernels are studied in \citep{gpucb}.
Specifically, \cite{gpucb} state that  ${\gamma _T} = \mathcal{O}(d\log T)$ for the linear kernel, ${\gamma _T} = \mathcal{O}({(\log T)^{d + 1}})$ for the squared
exponential  kernel and ${\gamma _T} = \mathcal{O}({T^\alpha }(\log T))$
for the Mat\'ern kernels with $\nu >1$, where $\alpha  = \frac{{d(d + 1)}}{{2\nu  + d(d + 1)}} \le 1$. We  employ the term $\gamma_T$ to build the regret bounds of our algorithms.

In this work, we consider two settings:  noise-free setting and  perturbation setting:

\noindent \textbf{Noise-Free Setting:}
We assume the underlying function $f$ belongs to an RKHS associated with kernel $k(\cdot,\cdot)$ , i.e.,  $f \in  \mathcal{H}_k$, with ${\left\| f \right\|_{{{\cal H}_{{k }}}}^{}} < \infty$ . In the noise-free setting, we can directly observe $f(x), x \in \mathcal{X}$ without noise perturbation. 

\noindent \textbf{Perturbation Setting:} In the perturbation setting, we cannot observe the function evaluation $f(x)$ directly. Instead, we observe $y = h(x)= f(x) + g(x)$, where $g(x)$ is an unknown perturbation function.

Define ${k^\sigma }(x,y): = k(x,y) + {\sigma ^2}\delta (x,y)$ for $x,y \in \mathcal{X}$, where $\delta (x,y) = \left\{ {\begin{array}{*{20}{c}}
{1\;\;x = y}\\
{0\;\;x \ne y}
\end{array}} \right.$ and $\sigma \ge 0$.  We assume $f \in \mathcal{H}_k$ ,  $g \in \mathcal{H}_{\sigma^2\delta}$ with ${\left\| f \right\|_{{{\cal H}_{{k }}}}^{}} < \infty$ and ${\left\| g \right\|_ { \mathcal{H}  _{\sigma^2\delta}}} < \infty$, respectively.
Therefore, we know $h \in \mathcal{H}_{k^\sigma}$ and ${\left\| h \right\|_ { \mathcal{H}  _{k^\sigma}}} < \infty$. The same point is assumed never selected twice, this can be ensured by the deterministic selection rule.

\section{BO in Noise-Free Setting}
\label{noisefreesetting}

In this section, we will first present algorithms and theoretical analysis in the sequential case. We then discuss our batch selection method. All detailed proofs are included in the supplementary material.

\subsection{Sequential Selection in Noise Free Setting}

 Define ${m_t}(x)$ and ${\sigma _t}(x)$ as  follows: 
 \begin{align}\label{m_t}
 & {m_t}(x) = {{\bf{k}}_t}(x)^T{\bf{K}}_t^ {-1} {{\bf{f}}_t} \\
 & {\sigma _t^2}(x) = k(x,x) - {{\bf{k}}_t}{(x)}^T{\bf{K}}_t^  {-1} {{\bf{k}}_t}(x) \label{sigmat}, 
 \end{align}
 where ${{\bf{k}}_t}(x)=[k(x,x_1),...,k(x,x_t)]^T $ and the kernel matrix ${{\bf{K}}_t} = {\left[ {k({x_i},{x_j})} \right]_{1 \le i,j \le t}}$. These terms are closely related to the posterior mean and variance functions of GP with zero  noise. We use them in the deterministic setting. A detailed review of the relationships between GP methods and kernel methods can be found in \citep{kanagawa2018gaussian}.

 The sequential optimization method in the noise-free setting is described in Algorithm~\ref{alg:1}. It has a similar form to  GP-UCB~\citep{gpucb}, except that it employs a constant weight of the term $\sigma _{t - 1}(x) $ to balance exploration and exploitation. In contrast, GP-UCB uses a $\mathcal{O}(\log(t))$ increasing weight. In practice, a constant weight is preferred in the scenarios where an aggressive selection manner is needed. For example, only a small number of evaluations can be done in the hyperparameter tuning in RL algorithms due to limited resources.
The regret bounds of Algorithm~\ref{alg:1} are given in Theorem~\ref{sequentialBO}. 
\begin{theorem}\label{sequentialBO}
Suppose $f \in \mathcal{H}_k$ associated with $k(x,x) \le 1$ and ${\left\| f \right\|_{{{\cal H}_{{k }}}}^{}} < \infty$. Let $C_1=\frac{8}{{\log (1 + {\sigma ^{ - 2}})}}$. 
\textbf{Algorithm 1} achieves a cumulative regret bound and a simple regret bound given as follows:
\begin{align}\label{R_Tseq}
{R_T} \le {\left\| f \right\|_{{\mathcal{H}_k}}}\sqrt {TC_1{\gamma _T}} \\ 
{r_T} \le {\left\| f \right\|_{{\mathcal{H}_k}}}\sqrt {\frac{{C_1{\gamma _T}}}{T}}.  \label{r_Tseq}
\end{align}
where $ 0< c < +\infty$.
\end{theorem}
\textbf{Remark:} 
\textcolor{black}{We can achieve  concrete bounds w.r.t $T$ by replacing $\gamma_T$ with the specific bound for the corresponding kernel. For example, for SE kernels, we can obtain that $R_T = \mathcal{O}(\sqrt{T}{(\log T)^{d + 1}})$  and $r_T = \mathcal{O}(\frac{{(\log T)^{d + 1}}}{\sqrt{T}} )$, respectively.  \cite{bull} presents   bounds  for Mat\'ern type kernels. The bound in  Theorem~\ref{sequentialBO} is tighter than Bull's bound of pure EI (Theorem 4 in~\citep{bull}) when the smoothness parameter of the Mat\'ern kernel  $\nu> \frac{d(d+1)}{d-2} = \mathcal{O}(d)$.  This is no better than the bound of mixed strategies (Theorem 5) in Bull's work.      Nevertheless, the bound in Theorem~\ref{sequentialBO} makes fewer  assumptions about the kernels, and covers more general results (kernels) compared with Bull's work.    }

\begin{algorithm}[t]
   \caption{}
   \label{alg:1}
\begin{algorithmic}

                                   
 \FOR{$t=1$ {\bfseries to} $T$}
  \STATE  Obtain ${m _{t - 1}}( \cdot )$ and $ \sigma _{t - 1}^2( \cdot )$ via equations (\ref{m_t}) and (\ref{sigmat}).
  \STATE Choose ${x_t} = \mathop {\arg \max }\nolimits_{x \in \mathcal{X} } {m _{t - 1}}(x) + {\left\| f \right\|_{{\mathcal{H}_k}}} \sigma _{t - 1}(x)  $. 
  
 
 \STATE Query the observation $f(x_t)$  at location $x_t$.  

   \ENDFOR
\end{algorithmic}
\end{algorithm}

\subsection{Batch Selection in Noise-Free Setting}

Let N and L be the number of batches and the batch size, respectively.
Without loss of generality, we assume  $T=NL$. Let $X_n=\{ {x_{(n-1)L+1},...,x_{nL} }\} $ and ${\overline X_{ n }} = \{X_1,...,X_n\}= \left\{ {{x_1},...,{x_{n L}}} \right\}$ be the $n^{th}$ batch of points and all the   $n$ batches of points,  respectively. The  covariance function   of $X \in \mathbb{R}^{ d \times L}$ for the noise free case is    given as follows:  
\begin{align}\label{BnoiseFree}
  &  {{\mathop{\rm cov}} _n}({X},{X}) =    {\bf{K}}({X},{X})  -  {\bf{K}}{({\overline X _n},{X} )^T}{{\bf{K}}({{\overline X }_n},{{\overline X }_n}) ^{ \!-\! 1}}{\bf{K}}({\overline X _n},{X}\! )
\end{align}
where $ \bf{K}({X},{X})$ is the $L \times L$ kernel matrix, $ {\bf{K}}{({\overline X _n},{X})} $ denotes the $ nL \times L$ kernel matrix between  ${\overline X _n}$ and $X$. When $n=0$,   $ {\mathop{\rm cov}} _0({X},{X})= {\bf{K}}({X},{X})$ is the prior kernel matrix. We assume that  the kernel matrix is invertible in the noise-free setting.

The proposed batch optimization algorithm is presented in Algorithm~\ref{alg:2}.  It employs the mean prediction value of a batch together with a term of covariance to balance the exploration/exploitation trade-off.  The covariance term in Algorithm~\ref{alg:2} penalizes the batch with over-correlated points. Intuitively, for SE kernels and Mat\'ern kernels, it penalizes the batch with points that are too close to each other (w.r.t Euclidean distance). As a result, it encourages the points in a batch to spread out for better exploration. 
The regret bounds of our batch optimization method are summarized in Theorem~\ref{batchBO}. 

\begin{theorem}\label{batchBO}
Suppose $f \in \mathcal{H}_k$ associated with $k(x,x) \le 1$ and ${\left\| f \right\|_{{{\cal H}_{{k }}}}^{}} < \infty$. Let  $T=NL$, $\beta = {\max _{n \in \{ 1,...,N\} }}{\left\| {\widehat { \mathop{\rm cov}} _{n - 1}({X_n},{X_n})} \right\|_2}$  and  $C_2=\frac{8\beta}{{\log (1 + \beta{\sigma ^{ - 2}})}}$ . 
\textbf{Algorithm 2} with batch size $L$ achieves a cumulative regret bound and a simple regret bound given by equations (\ref{R_Tseq2}) and (\ref{r_Tseq2}), respectively:
\begin{align}\label{R_Tseq2}
 {R_T} \le {\left\| f \right\|_{{\mathcal{H}_k}}}\sqrt {TC_2{\gamma _{T}} } \\ 
 {r_T} \le {\left\| f \right\|_{{\mathcal{H}_k}}}\sqrt {\frac{{C_2{\gamma _{T}}}}{T}}.  \label{r_Tseq2}
\end{align}
\end{theorem}
\textbf{Remark:} (1) A large $\beta$ leads to a large bound, while a small $\beta$ attains a small bound. Algorithm 2 punishes the correlated points and encourages the uncorrelated points in a batch, which can attain a small $\beta$ in general.   (2) A trivial bound of $\beta $ is $\beta \le L$.

To prove Theorem \ref{batchBO}, the following Lemma is proposed. The detailed proof can be found in the supplementary material.
\begin{lemma}
\label{RKHSbound1}
Suppose $f \in \mathcal{H}_k$ associated with kernel $k(x,x)$ and ${\left\| f \right\|_{{{\cal H}_{{k }}}}^{}} < \infty$, then we have ${\left( {\sum\nolimits_{i = 1}^L {{m_t}({{\widehat x_i}}) - \sum\nolimits_{i = 1}^L {f({{\widehat x_i}})} } } \right)^2} \le \left\| f \right\|_{{\mathcal{H}_k}}^2  ({{\bf{1}}^T}{\bf{A1}})$, where $\bf{A}$ denotes the kernel  covariance matrix  with ${{\bf{A}}_{ij}} = k({\widehat x_i},{\widehat x_j}) - {{\bf{k}}_t}{({\widehat x_i})^T}{\bf{K}}_t^ - {{\bf{k}}_t}({\widehat x_j})$.
\end{lemma}
\textbf{Remark:} Lemma~\ref{RKHSbound1} provides a tighter bound for the deviation of the summation of a batch  than directly applying the bound for a single point $L$ times.

\begin{algorithm}[t]
   \caption{}
   \label{alg:2}
\begin{algorithmic}
   
                                   
  \FOR{$n=1$ {\bfseries to} $N$}
  \STATE  Obtain ${m _ {(n-1)L}}( \cdot )$ and ${\mathop{\rm cov}} _{n-1}(\cdot)$  via equations (\ref{m_t}) and (\ref{BnoiseFree}),  respectively.
  \STATE Choose ${X_n} \! = \! \mathop {\arg \max }\limits_{X \subset {\cal X}} \frac{1}{L}\sum\limits_{i = 1}^L {{m_{(n - 1)L}}({X_{ \cdot ,i}})}  + {\left\| f \right\|_{{{\cal H}_k}}}\!\!\left(\! {2\sqrt {\frac{{tr({{{{\mathop{\rm cov}} }_{n - 1}}(X,X)})}}{L}}  -\! \sqrt {\frac{{{{\bf{1}}^T}{{{{\mathop{\rm cov}} }_{n - 1}}(X,X)}{{\bf{1}}}}}{{{L^2}}}} } \right) $.
 
 \STATE Query the batch observations $\{f(x_{(n-1)L+1}),...,f(x_{nL}) \}$ at locations $ X_n $.

   \ENDFOR
\end{algorithmic}
\end{algorithm}

\section{BO in Perturbation Setting}
\label{perturbationsetting}

In the perturbation setting, we cannot observe the function evaluation $f(x)$ directly. Instead, we observe $y = h(x)= f(x) + g(x)$, where $g(x)$ is an unknown perturbation function.
We will discuss the sequential selection and batch selection methods in the following sections,  respectively.

\subsection{Sequential Selection in Perturbation Setting}

Define ${\widehat m_t}(x)$ and ${\widehat \sigma _t}(x)$ as  follows: 
\begin{align}\label{m_tN}
& {\widehat m_t}(x) = {{\bf{k}}_t}(x)^T({\bf{K}}_t + \sigma^2I)^  {-1} {{\bf{y}}_t} \\
& {\widehat \sigma _t^2}(x) = k(x,x) - {{\bf{k}}_t}{(x)}^T({\bf{K}}_t + \sigma^2I)^  {-1} {{\bf{k}}_t}(x), \label{sigmatN}
\end{align}
where ${{\bf{k}}_t}(x)=[k(x,x_1),...,k(x,x_t)]^T $ and the kernel matrix ${{\bf{K}}_t} = {\left[ {k({x_i},{x_j})} \right]_{1 \le i,j \le t}}$.

The sequential selection method is presented in Algorithm~\ref{alg:3}. It has a similar formula to  Algorithm~\ref{alg:1}; while Algorithm~\ref{alg:3} employs a regularization $\sigma^2I$  to handle the uncertainty of the perturbation. The regret bounds of Algorithm~\ref{alg:3} are summarized in Theorem~\ref{TheoremNS}.

\begin{algorithm}[t]
   \caption{}
   \label{alg:3}
\begin{algorithmic}

                                   
 \FOR{$t=1$ {\bfseries to} $T$}
  \STATE  Obtain ${\widehat m _{t - 1}}( \cdot )$ and $ \widehat \sigma _{t - 1}^2( \cdot )$ via equation (\ref{m_tN}) and (\ref{sigmatN}).
  \STATE Choose ${x_t} = \mathop {\arg \max }\nolimits_{x \in \mathcal{X} } {\widehat m _{t - 1}}(x) + {\left\| h \right\|_{{\mathcal{H}_{k^\sigma}}}} \widehat \sigma _{t - 1}(x)  $ 
  
 
 \STATE Query the  observation $y_t=h(x_t)$  at location $x_t$.  

   \ENDFOR
\end{algorithmic}
\end{algorithm}

\begin{theorem}
\label{TheoremNS}
Define ${k^\sigma }(x,y): = k(x,y) + {\sigma ^2}\delta (x,y) \le B$, where $\delta (x,y) = \left\{ {\begin{array}{*{20}{c}}
{1\;\;x = y}\\
{0\;\;x \ne y}
\end{array}} \right.$ and $\sigma \ge 0$. 
Suppose $f \in \mathcal{H}_k$,  $g \in \mathcal{H}_{\sigma^2 \delta}$ associated with kernel $k$ and kernel $\sigma^2 \delta $ with ${\left\| f \right\|_{{{\cal H}_{{k }}}}^{}} < \infty$ and ${\left\| g \right\|_ { \mathcal{H}  _{\sigma^2\delta}}} < \infty$, respectively. Let $C_3=\frac{8B}{{\log (1 + B {\sigma ^{ - 2}})}}$. 
\textbf{Algorithm \ref{alg:3}} achieves a cumulative regret bound and a simple regret bound given by equations (\ref{nR_Tseq}) and (\ref{nr_Tseq2}), respectively.

\begin{align}
\label{nR_Tseq}
{R_T} \le {\left\| h \right\|_{{\mathcal{H}_{k^\sigma}}}}\!\!\!\sqrt {T{C_3}{\gamma _T}} \!  + \!\! 2T\left(\! {\left\| h \right\|_{{{\cal H}_{{k^\sigma }}}}^{} \! + \! \left\| g \right\|_{{{\cal H}_{{\sigma ^2}\delta }}}^{}} \!\right)\!{\sigma ^{}}  \\ 
{r_T} \le {\left\| h \right\|_{{\mathcal{H}_{k^\sigma}}}}\!\!\sqrt {\frac{{C_3{\gamma _T}}}{T}}   + 2\left( {\left\| h \right\|_{{{\cal H}_{{k^\sigma }}}}^{} + \left\| g \right\|_{{{\cal H}_{{\sigma ^2}\delta }}}^{}} \right){\sigma ^{}}  \label{nr_Tseq2}
\end{align} 
\end{theorem}
\textbf{Remark:} In the perturbation setting, the unknown perturbation function $g$ results in some unavoidable dependence on $\sigma$  in the regret bound compared with GP-UCB~\citep{gpucb}.  Note that the bounds in \citep{gpucb} are probabilistic bounds. There is always a positive probability that the bounds in \citep{gpucb} fail. In contrast, the bounds in Theorem~\ref{TheoremNS} are deterministic. 


\begin{corollary}
\label{CorollaryBiased}
Suppose $h=f \in \mathcal{H}_k$ associated with ${k}(x,y) \le 1$ and ${\left\| f \right\|_{{{\cal H}_{{k }}}}^{}} < \infty$.  Let $C_1=\frac{8}{{\log (1 +  {\sigma ^{ - 2}})}}$. 
\textbf{Algorithm \ref{alg:3}} achieves a cumulative regret bound and a simple regret bound given by equations (\ref{MnR_Tseq}) and (\ref{Mnr_Tseq2}), respectively:
\begin{align}
\label{MnR_Tseq}
{R_T} \le {\left\| f \right\|_{{{\cal H}_{{k }}}}^{}}\sqrt {T{C_1}{\gamma _T}}   + 2T{\left\| f \right\|_{{{\cal H}_{{k }}}}^{}}\sigma   \\ 
{r_T} \le {\left\| f \right\|_{{{\cal H}_{{k }}}}^{}}\sqrt {\frac{{C_1{\gamma _T}}}{T}}   + 2 {\left\| f \right\|_{{{\cal H}_{{k }}}}^{}}\sigma .  \label{Mnr_Tseq2}
\end{align}   
\end{corollary}
\begin{proof}
Setting $g=0$ and $B=1$ in Theorem \ref{TheoremNS}, we can achieve the results.
\end{proof}
\textbf{Remark:} In practice, a small constant $\sigma^2I$ is added to the kernel matrix to avoid numeric problems in the noise-free setting. Corollary \ref{CorollaryBiased} shows that the small constant results in an additional biased term in the regret bound. Theorem \ref{sequentialBO} employs (\ref{m_t}) and (\ref{sigmat}) for  updating, while Corollary \ref{CorollaryBiased} presents the regret bound for the practical updating by (\ref{m_tN}) and (\ref{sigmatN}).

\subsection{Batch Selection in Perturbation Setting}

The covariance kernel function of $X \in \mathbb{R}^{ d \times L}$ for the perturbation setting is defined as equation (\ref{NoiseBatchV}),
\begin{equation}
\label{NoiseBatchV}
\begin{array}{l}
{\widehat { \mathop{\rm cov}} _n}({X},{X}) = {\bf{K}}({X},{X}) -  {\bf{K}}{({\overline X _n},{X} )^T}{\left( {{\sigma ^2}I + {\bf{K}}({{\overline X }_n},{{\overline X }_n})} \right)^{ - 1}}{\bf{K}}({\overline X _n},{X} ),
\end{array}
\end{equation}
where $ \bf{K}({X},{X})$ is the $L \times L$ kernel matrix, and $ {\bf{K}}{({\overline X _n},{X})} $ denotes the $ nL \times L$ kernel matrix between  ${\overline X _n}$ and $X$. 
The batch optimization method for the perturbation setting is presented in Algorithm~\ref{alg:4}.  The regret bounds of Algorithm~\ref{alg:4} are summarized in Theorem~\ref{batchBOnoise}. 
\begin{theorem}\label{batchBOnoise}
Define ${k^\sigma }(x,y): = k(x,y) + {\sigma ^2}\delta (x,y) \le B$, where $\delta (x,y) = \left\{ {\begin{array}{*{20}{c}}
{1\;\;x = y}\\
{0\;\;x \ne y}
\end{array}} \right.$ and $\sigma \ge 0$. 
Suppose $f \in \mathcal{H}_k$ and  $g \in \mathcal{H}_{\sigma^2 \delta}$ associated with kernel $k$ and kernel $\sigma^2 \delta $ with ${\left\| f \right\|_{{{\cal H}_{{k }}}}^{}} < \infty$ and ${\left\| g \right\|_ { \mathcal{H}  _{\sigma^2\delta}}} < \infty$, respectively. 
 Let  $T=NL$, $\beta = {\max _{n \in \{ 1,...,N\} }}{\left\| { \widehat {\mathop{\rm cov}} _{{n - 1}}({X_n},{X_n})} \right\|_2}$  and  $C_4=\frac{8\beta}{{\log (1 + \beta{\sigma ^{ - 2}})}}$ . 
\textbf{Algorithm \ref{alg:4}} with batch size $L$ achieves a cumulative regret bound and a simple regret bound given by equations (\ref{R_TBN}) and (\ref{r_TBN}), respectively:
\begin{align}\label{R_TBN}
 {R_T} \le {\left\| h \right\|_{{\mathcal{H}_{k^\sigma}}}}\!\!\sqrt {TC_4{\gamma _{T}} } \!  + \! 2T\!\left(\! {\left\| h \right\|_{{{\cal H}_{{k^\sigma }}}}^{} \!\!+\! \left\| g \right\|_{{{\cal H}_{{\sigma ^2}\delta }}}^{}}\! \right)\!{\sigma ^{}} \\ 
 {r_T} \le {\left\| h \right\|_{{\mathcal{H}_{k^\sigma}}}}\!\!\sqrt {\frac{{C_4{\gamma _{T}}}}{T}} +  2\left( {\left\| h \right\|_{{{\cal H}_{{k^\sigma }}}}^{} + \left\| g \right\|_{{{\cal H}_{{\sigma ^2}\delta }}}^{}} \right){\sigma ^{}}. \label{r_TBN}
\end{align}
\end{theorem}
\textbf{Remark:} When the batch size is one, the regret bounds reduce to the sequential case.

\begin{figure*}[t!]
\centering
\subfigure[\scriptsize{Rosenbrock function}]{
\label{fig2a}
\includegraphics[width=0.31\linewidth]{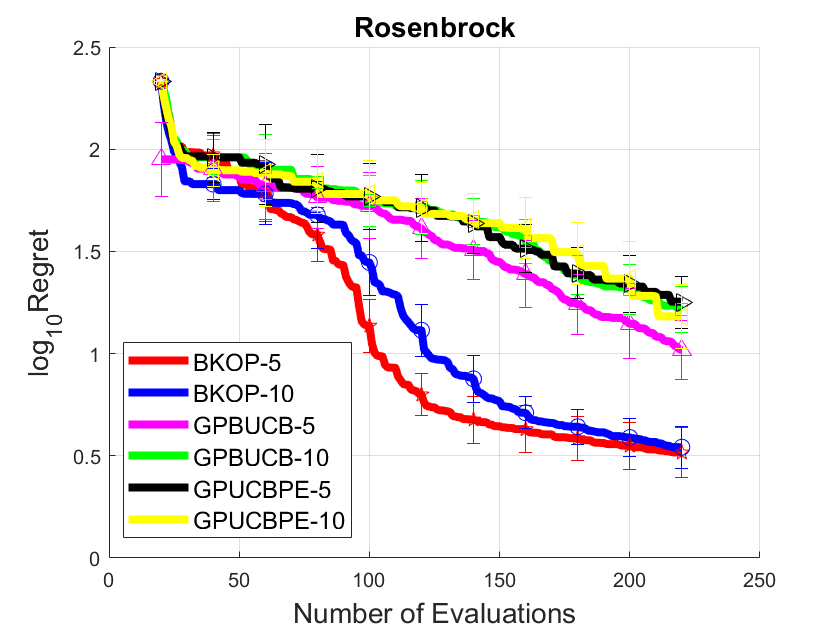}}
\subfigure[\scriptsize{Nesterov function}]{
\label{fig2c}
\includegraphics[width=0.31\linewidth]{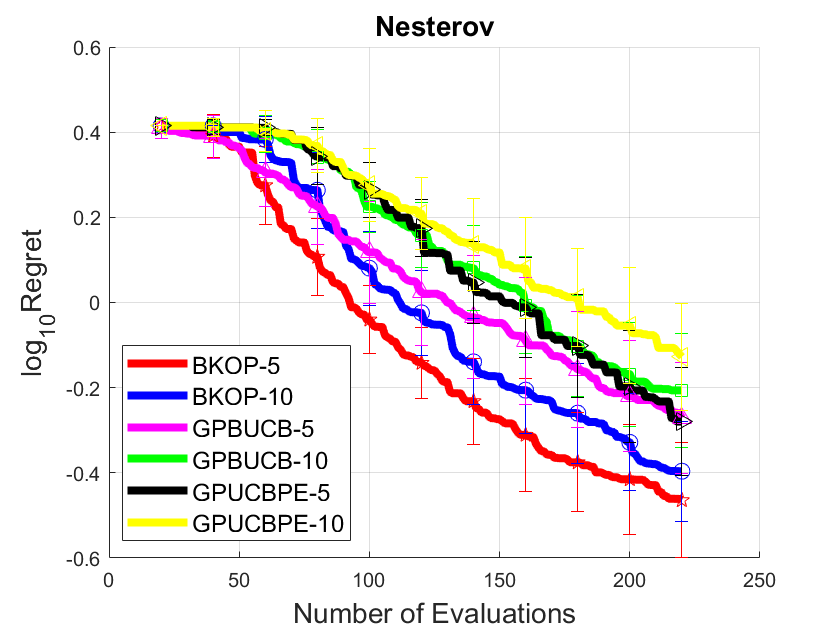}}
\subfigure[\scriptsize{Different-Powers function}]{
\label{fig2f}
\includegraphics[width=0.31\linewidth]{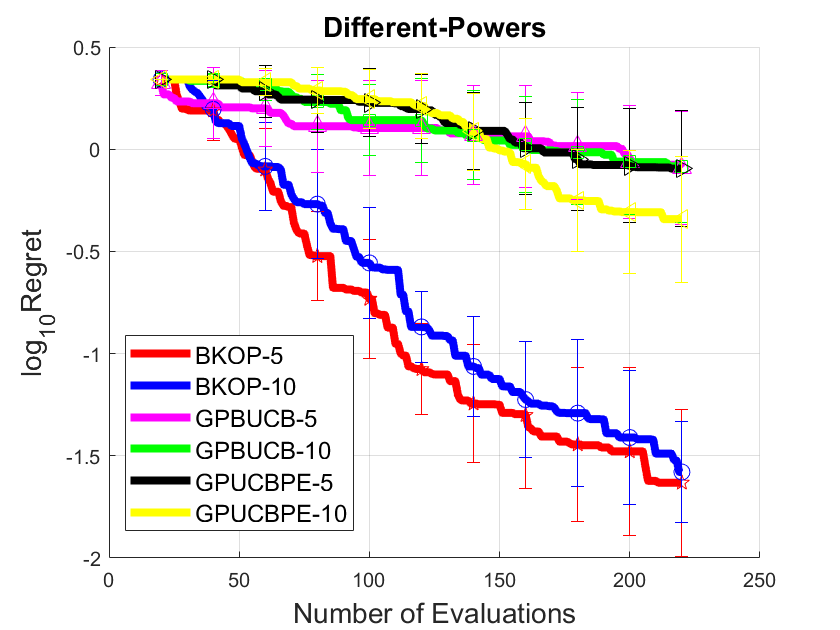}}
\subfigure[\scriptsize{Dixon-Price function}]{
\label{fig2b}
\includegraphics[width=0.31\linewidth]{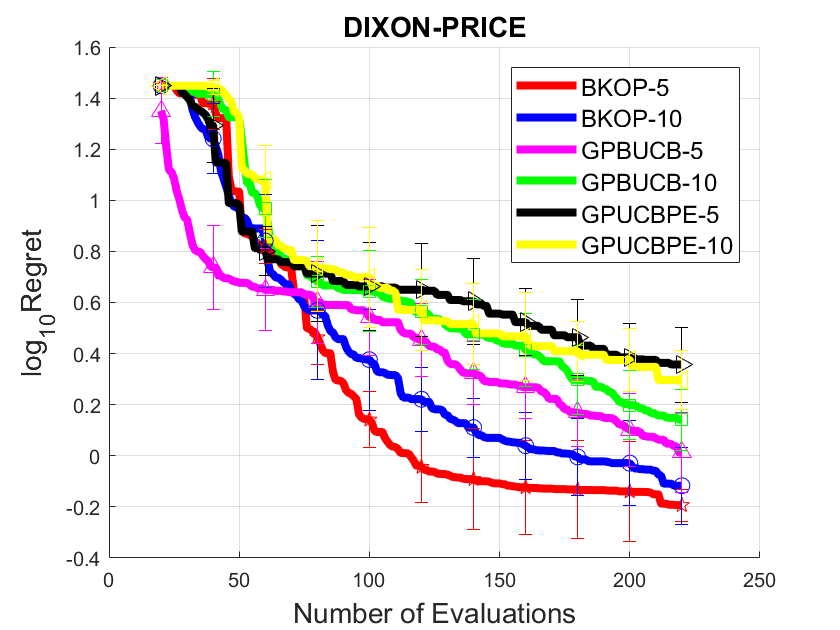}}
\subfigure[\scriptsize {Levy function} ]{
\label{fig2d}
\includegraphics[width=0.31\linewidth]{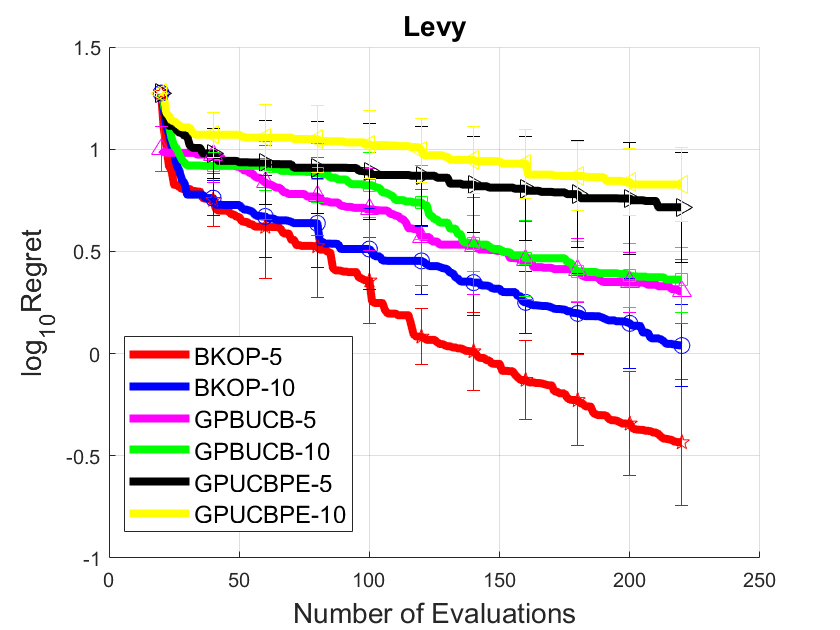}}
\subfigure[\scriptsize{Ackley  function}]{
\label{fig2g}
\includegraphics[width=0.31\linewidth]{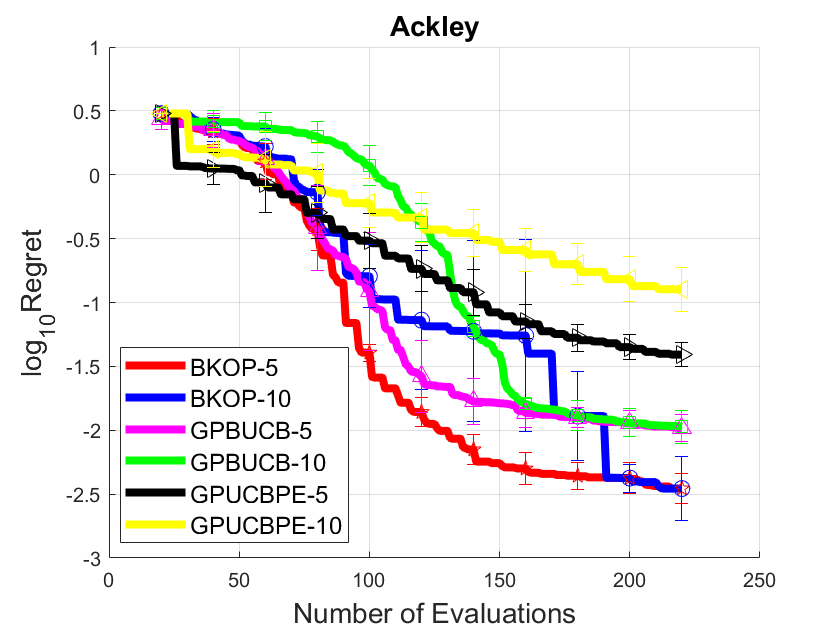}}
\caption{The mean value of simple regret for different algorithms over 30 runs on different test functions}
\label{fig1}
\end{figure*}

\begin{algorithm}[t]
   \caption{}
   \label{alg:4}
\begin{algorithmic}
   
                                   
  \FOR{$n=1$ {\bfseries to} $N$}
  \STATE  Obtain ${\widehat m _ {(n-1)L}}( \cdot )$ and $ \widehat {\mathop{\rm cov}} _{n-1}(\cdot)$  via equation (\ref{m_tN}) and (\ref{NoiseBatchV}) respectively.
  \STATE Choose ${X_n} \!\! = \! \mathop {\arg \max }\limits_{X \subset {\cal X}} \!\frac{1}{L}\sum\limits_{i = 1}^L { {\widehat m_{(n - 1)L}}({X_{ \cdot ,i}})}  + {\left\| h \right\|_{{\mathcal{H}_{k^\sigma}}}}\!\!\left(\! {2\sqrt {\frac{{tr({{{ \widehat {  \mathop{\rm cov}} }_{n - 1}}(X,X)})}}{L}} \!  -  \!\!\sqrt {\frac{{{{\bf{1}}^T}{{{\widehat  {\mathop{\rm cov}} }_{n - 1}}(X,X)}{{\bf{1}}}}}{{{L^2}}}} } \right)$.
 
 \STATE Query the batch observations $\{h(x_{(n-1)L+1}),...,h(x_{nL}) \}$ at locations $ X_n = \{x_{(n-1)L+1},...,x_{nL} \} $.

   \ENDFOR
\end{algorithmic}
\end{algorithm}

\color{black}
\section{ Robust Initialization for BO}\label{advSet}

In practice, the initialization phase of BO is  important. 
In this section, we will discuss how to achieve robust initialization by analyzing regret in the adversarial setting. We will first show that algorithms that attain a  small covering radius (fill distance) can achieve small adversarial regret bounds. Based on this insight, we provide a robust initialization to BO.


Let  $f_t: \mathcal{X} \to \mathbb{R}$, $t \in [T]$ be the   black-box function to be optimized at round $t$. Let $f_t(x^*_t)={\max _{x \in \mathcal{X} }}f_t(x)$ with $f_t \in \mathcal{B}_k $.
The simple adversarial regret $ {\widetilde{r}_T}  $ is defined as:
\begin{align}\label{AsimpleR}
& {\widetilde{r}_T} = {\min _{t \in [T] }}\sup_{ \substack{ f_t\in \mathcal{B}_k, \forall {i}  \in [t-1], \\ f_t(x_i)=f_i(x_i) }  }\{f_t(x^*_t) - f_t({x_t}) \},
\end{align}
where the constraints ensure that each $f_t$ has the same observation values as the history at previous query points $X_{t-1}=\{x_1,...,x_{t-1}\}$. 
This can be viewed as an adversarial game. During each round $t$, the opponent chooses a function $f_t$ from a candidate set, and we then choose a query  $x_t$ to achieve a small regret. A robust initialization setting can be viewed as the batch of points that can achieve a low simple adversarial regret irrespective of the access order.


Define covering radius (fill distance~\citep{kanagawa2018gaussian}) and  packing radius (separate distance~\citep{kanagawa2018gaussian}) of a set of points $X=\{x_1,...,x_T\}$ as follows:
\begin{align}\label{hx}
     h_X =  \sup_{x\in \mathcal{X}}{\min _{x_t \in X }} \|x-x_t \|  \\
     \rho_X = \frac{1}{2} \min _ {\substack{x_i,x_j  \in X, \\ x_i \ne x_j}} \|x_i -x_j \| . \label{rx}
\end{align}
 Our method for robust initialization is presented in Algorithm~\ref{alg:mmd}, which  constructs an initialization set $X_{T\!-1}$ by minimizing  the covering radius.   We present one such method in Algorithm~\ref{alg:6} in the next section. The initialization set $X_{T\!-1}$ can be evaluated in a batch manner, which is able to benefit from parallel evaluation.
The regret bounds of Algorithm~\ref{alg:mmd} are summarized in Theorem~\ref{Sobolev} and Theorem~\ref{SE}. 



\begin{theorem}\label{Sobolev}
Define $\mathcal{B}_k= \{f: f \in \mathcal{H}_k, {\left\| f \right\|_{{\mathcal{H}_k}}} \le B\}$ associated with $k(x,x) $ for $x \in \mathcal{X} \subset \mathbb{R}^d $. Suppose $f \in \mathcal{B}_k$   and $\mathcal{H}_k$ is norm-equivalent to the Sobolev space of order $ s $.  Then there exits a constant $C>0$, such that the query point set generated by  \textbf{Algorithm~5}  with a sufficiently small covering radius (fill distance) $h_X$ achieves a regret bound given by equation (\ref{Mr_Tseq}): 
\begin{align}\label{Mr_Tseq}
\widetilde{r}_T \le BCh_X^{s-d/2} .
\end{align} 
\end{theorem}
\textbf{Remark:} The regret bound decreases as the covering radius becomes smaller. This means that a query set with a small covering radius can guarantee a small regret.  \cite{bull} gives  bounds of fixed points set for  Mat\'ern kernels (Theorem 1). However, it does not link to the covering radius. The bound in Theorem~\ref{Sobolev} directly links to the covering radius, which provides a criterion for generating points to achieve small bounds.


\begin{theorem}\label{SE}
Define $\mathcal{B}_k= \{f: f \in \mathcal{H}_k, {\left\| f \right\|_{{\mathcal{H}_k}}} \le B\}$ associated with  square-exponential $k(x,x) $ on unit cube $ \mathcal{X} \subset \mathbb{R}^d $. Suppose $f \in \mathcal{B}_k$.  Then there exits a constant $c>0$,  such that the query point set generated by  \textbf{Algorithm~5} 
 with a sufficiently small covering radius (fill distance) $h_X$ achieves a regret bound given by equation (\ref{EMr_Tseq}): 
\begin{align}\label{EMr_Tseq}
{\widetilde{r}_T} \le B\exp(c\log(h_X)/(2\sqrt{h_X})).
\end{align} 
\end{theorem}
\textbf{Remark:} Theorem~\ref{SE} presents a regret bound for the SE kernel. It attains higher rate w.r.t covering radius $h_X$ compared with Theorem~\ref{Sobolev}, because functions in RKHS  with SE kernel are more smooth than functions in Sobolev space.

\begin{algorithm}[b]
   \caption{}
   \label{alg:mmd}
\begin{algorithmic}

                                   
 \STATE Construct Candidate set $X_{T-1}$ with $T\!-\!1$ points by minimizing the fill distance (e.g.Algorithm \ref{alg:6}).
 \STATE Query the  observations at $X_{T-1}$.

  \STATE  Obtain ${m _{T - 1}}( \cdot )$ and $ \sigma _{T - 1}^2( \cdot )$ via equation (\ref{m_t}) and (\ref{sigmat}).
  \STATE Choose ${x_T} = \mathop {\arg \max }\limits_{x \in \mathcal{X} } { m _{T - 1}}(x) + B \sigma _{T - 1}(x)$ 
  
 
 \STATE Query the  observation $y_T=f(x_T)$  at location $x_T$.  

\end{algorithmic}
\end{algorithm}

\begin{figure}[t]
\centering
\subfigure[\scriptsize{100 lattice points }]{
\label{fig2a_l}
\includegraphics[width=0.45\linewidth]{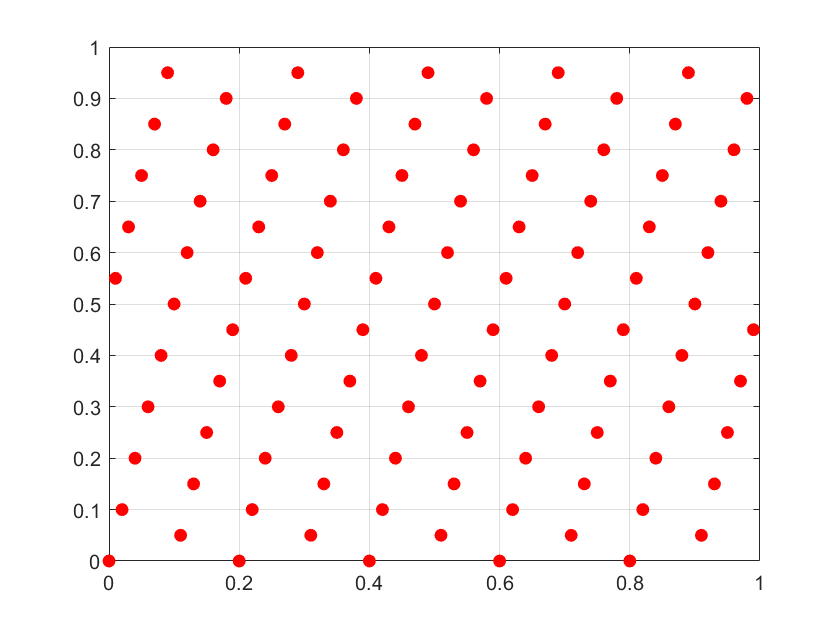}}
\subfigure[\scriptsize{100 random points }]{
\label{fig2c_l}
\includegraphics[width=0.45\linewidth]{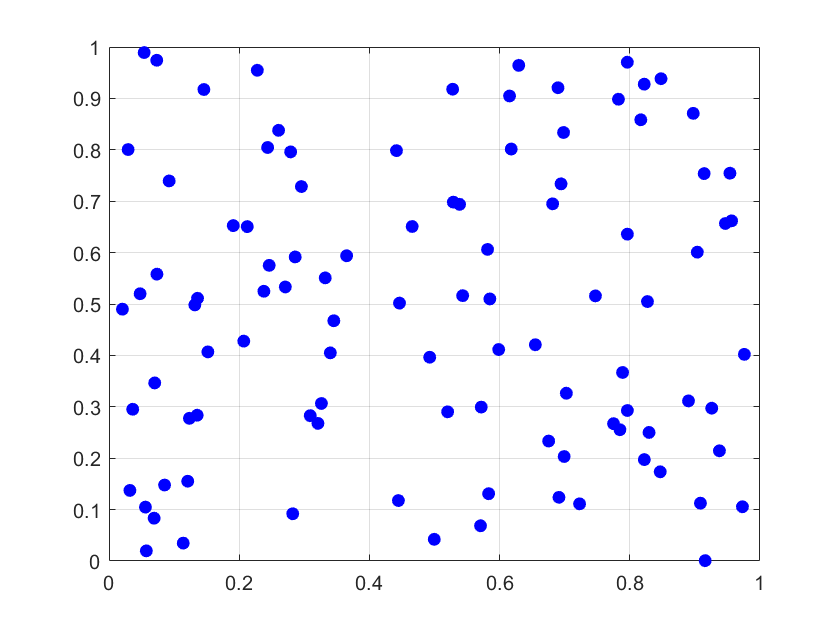}}
\caption{Lattice Points and Random Points on $[0,1]^2$}
\label{fig_lattice}
\end{figure}


We analyze the regret under a more adversarial setting. This relates to a more robust requirement.
The regret bounds under a fully adversarial setting when little information is known are summarized in Theorem~\ref{FullyAdv}.



\begin{theorem}\label{FullyAdv}
Define $\mathcal{B}_k= \{f: f \in \mathcal{H}_k, {\left\| f \right\|_{{\mathcal{H}_k}}} \le B\}$ associated with a shift invariant kernel $k(x,y) =\Phi(\|x-y \|) \le 1$ that decreases w.r.t $\|x-y \|$. Suppose  $    \exists x^*$ such that $f_t(x^*)={\max _{x \in \mathcal{X} }}f_t(x)$ with $f_t \in \mathcal{B}_k $  for $t \in [T]$. Then the query point set $X=\{x_1,...,x_T\}$ generated by  \textbf{Algorithm~5} 
 with covering radius (fill distance) $h_X$ achieves  a  regret bound~as
\begin{align}
   {\bar{r}_T} \!=  \! {\min _{t \in \{1,...,T\} }}\sup_{f_t\in \mathcal{B}_k}\{f_t(x^*) \!-\! f_t({x_t}) \} \!\le\! B\! \sqrt{2\!-\!2\Phi(h_X)}. \nonumber
\end{align}
\end{theorem}
\textbf{Remark:} Theorem 9 gives a fully adversarial bound. Namely, the opponent can choose functions from $\mathcal{B}_k$ without the same history. The regret bound decreases with the decrease of the covering radius (fill distance). The assumption requires each $f_t$ to have the $x^*$ as one of its maximum. Particularly, it is satisfied when $f_1\!=\!\cdots\!=f_T =\! f$.

\begin{corollary}\label{SEbound}
Define $\mathcal{B}_k= \{f: f \in \mathcal{H}_k, {\left\| f \right\|_{{\mathcal{H}_k}}} \le B\}$ associated with squared exponential kernel. 
 Suppose  $    \exists x^*$ such that $f_t(x^*)={\max _{x \in \mathcal{X} }}f_t(x)$ with $f_t \in \mathcal{B}_k $  for $t \in [T]$. Then the query point set $X=\{x_1,...,x_T\}$ generated by  \textbf{Algorithm~5} 
 with covering radius (fill distance) $h_X$ achieves  a  regret bound~as
\begin{align}
    & {\bar{r}_T} = {\min _{t \in \{ 1,...,T\} }}\sup_{f_t\in \mathcal{B}_k}\{f_t(x^*) - f_t({x_t}) \} \le \mathcal{O}(h_X).
\end{align}
\end{corollary}
\textbf{Remark:} For a regular grid, $h_X = \mathcal{O}(T^{-\frac{1}{d}})$ \citep{wendland2004scattered}, we then achieve ${\bar{r}_T} = \mathcal{O}(T^{-\frac{1}{d}})$. 
Computer search can find a point set with a smaller covering radius than that of a regular grid.

All the adversarial regret bounds discussed above decrease with the decrease of the covering radius. Thus, the point set generated by Algorithm~\ref{alg:mmd} with a small covering radius can serve as a good robust initialization for BO. 

\section{Fast Rank-1 Lattice Construction}
\label{rank-1cons}


  

\begin{algorithm}[t]
   \caption{Rank-1 Lattice Construction}
   \label{alg:6}
\begin{algorithmic}
    \STATE {\bfseries Input:} Number of primes  $M$, dimension $d$, number of lattice points $N$
 \STATE {\bfseries Output:} Lattice points $ X^{*} $, base vector $\bf{b^{*}}$ 
\STATE  Set $p_0= 2\times d+1$, initialize $\rho^{*}=-1$.
\STATE  Construct set $U:=\{p|p \in \mathbb{P}, p \ge p_0 \}$ containing $ M $ primes.

   \FOR{ each $p \in U$ }
 \FOR{$i=0$ {\bfseries to} $p-1$}
  \STATE Set ${\bf{g}} = \text{mod}({\bf{q}}+i,p) $, where $\bf{q}\in \mathbb{R}^{d-1}$ and $ {\bf{q}}_j \! = \!j $. 
  \STATE Set ${\bf{g}} = \text{round}(N\times \text{mod}(|2cos(\frac{2\pi{\bf{g}}}{p})|,1))$.
  \STATE Set $\bf{b}$ as $[1,{\bf{g}}]$ by concatenating vector $ 1 $ and $\bf{g}$. 
  \STATE Generate lattice $ X $ given base vector $\bf{b}$ as Eq.(\ref{Rank1}).
  \STATE  Calculate the packing radius (separate distance) $\rho_{X}$ of $ X $ as Eq.(\ref{rank1D}).
  \IF{ $ \rho_X > \rho^{*} $}
  \STATE Set ${\bf{b^{*}}}={\bf{b}}$ and  $\rho^{*} = \rho_X$.
  
  \ENDIF
   \ENDFOR
   \ENDFOR
   \STATE Generate lattice $X^{*}$ given base vector $\bf{b^{*}}$ as Eq.(\ref{Rank1}).
\end{algorithmic}
\end{algorithm}

\begin{algorithm}[t]
   \caption{Rank-1 Lattice Construction with Successive Coordinate Search (SCS) }
   \label{alg:7}
\begin{algorithmic}
    \STATE {\bfseries Input:} Number of primes  $M$, dimension $d$, number of lattice points $N$, number of iteration of SCS search subroutine $T$.
 \STATE {\bfseries Output:} Lattice points $ X^{*} $, base vector $\bf{b^{*}}$ 
\STATE  Set $p_0= 2\times d+1$, initialize $\rho^{*}=-1$.
\STATE  Construct set $U:=\{p|p \in \mathbb{P}, p \ge p_0 \}$ containing $ M $ primes.

   \FOR{ each $p \in U$ }
 \FOR{$i=0$ {\bfseries to} $p-1$}
  \STATE Set ${\bf{g}} = \text{mod}({\bf{q}}+i,p) $, where $\bf{q}\in \mathbb{R}^{d-1}$ and $ {\bf{q}}_j \! = \!j $. 
  \STATE Set ${\bf{g}} = \text{round}(N\times \text{mod}(|2cos(\frac{2\pi{\bf{g}}}{p})|,1))$.
  \STATE Set $\bf{b}$ as $[1,{\bf{g}}]$ by concatenating vector $ 1 $ and $\bf{g}$ . 
  \STATE  Perform SCS search~\citep{lyu2017spherical,SCSsearch} 
  with $\bf{b}$ as the initialization base vector to get a better base  $\bf{\widehat{b}}$ and $\rho_X$.
  \IF{ $ \rho_X > \rho^{*} $}
  \STATE Set ${\bf{b^{*}}}={\bf{\widehat{b}}}$ and  $\rho^{*} = \rho_X$.
  
  \ENDIF
   \ENDFOR
   \ENDFOR
   \STATE Generate lattice $X^{*}$ given base vector $\bf{b^{*}}$ as Eq.(\ref{Rank1}).
\end{algorithmic}
\end{algorithm}


In this section, we describe the procedure of generating a query points set that has a small covering radius (fill distance). Since minimizing the covering radius of the lattice is equivalent to maximizing the packing radius (separate distance)  \citep{keller2007monte}, we generate the query points set through maximizing the packing radius (separate distance) of the rank-1 lattice. An illustration of the  rank-1 lattice constructed by Algorithm~\ref{alg:6} is given in Fig.~\ref{fig_lattice}





\subsection{The rank-1 lattice construction given a base vector}
Rank-1 lattice is widely used in the Quasi-Monte Carlo (QMC) literature for integral approximation \citep{keller2007monte, Korobov}.
The lattice points of the rank-1 lattice in $ [0,1]^d $ are  generated by a base vector. Given an integer base vector $ {\bf{b}} \in \mathbb{N}^d $, a lattice set $ X $ that consists of $ N $ points in $ [0,1]^d $ is constructed~as 
\begin{align}\label{Rank1}
      X := \{{\bf{x}}_i:=  \text{mod}(i\times {\bf{b}},N)/N| i\! \in \!\{0,...,N\!\!-\!\!1\} \},
\end{align}
where $ \text{mod}(a,b)$ denotes the component-wise modular function, i.e., $a \% b$. We use $\text{mod}(a,1)$ to denote the fractional  part of number $a$ in this work.

\subsection{The separate distance of a rank-1 lattice}
Denote the toroidal distance \citep{minTdistance} between two lattice points $ {\bf{y}}\in [0,1]^d $ and $ {\bf{z}}\in [0,1]^d $ as: 
\begin{align}
    \|{\bf{y}} - {\bf{z}}\|_T := \!\!\sqrt{ \sum_{i=1}^d (\min( |y_i - z_i|  , 1-|y_i - z_i|))^2 }.
\end{align}
Because the difference (subtraction) between two lattice points is still a lattice point, and a rank-1 lattice has a periodic 1, the packing radius (separate distance) $ \rho_X $ of a rank-1 lattice with set $ X $ in $[0,1]^d$ can be calculated as
\begin{align}\label{rank1D}
    \rho_X = \min_{{\bf{x}} \in {X \setminus  {\bf{0}}}}\frac{1}{2}\|{\bf{x}} \|_T,
\end{align}
where $ \|{\bf{x}} \|_T $ can be seen as the toroidal distance between $ \bf{x} $ and $ \bf{0} $.
This formulation calculates the packing radius (separate distance) with a time complexity of $ \mathcal{O}(Nd) $ rather than $ \mathcal{O}(N^2d) $ in pairwise computation.





\subsection{Searching the  rank-1 lattice with maximized separate distance}
Given the number of primes $ M $, the dimension $ d $, and the number of lattices points $ N $, we try to find the optimal base vector $ b^{*} $ and its corresponding lattice points $ X^{*} $ such that the separation distance $ \rho_{X^{*}} $ is maximized over a candidate set. We adopt the algebra field based construction formula in \citep{hua2012applications} to construct the base vector of a rank-1 lattice. Instead of using the same predefined form as \citep{hua2012applications}, we adopt a searching procedure as summarized in Algorithm~\ref{alg:6}. The main idea is a greedy search starting from a set of $ M $ prime numbers. For each prime number $ p $, it also searches the $ p $ offset from $ 0 $ to $ p-1 $ to construct the possible base vector $ b $ and its corresponding $ X $. After the greedy search procedure, the algorithm returns the optimal base vector $ b^{*} $ and the lattice points set $ X^{*} $ that obtains the maximum separation distance. Algorithm~\ref{alg:6} can be extended by including successive coordinate search (SCS)~\citep{lyu2017spherical,SCSsearch} as an inner searching procedure. The extended method is summarized in  Algorithm~\ref{alg:7}. This method can achieve superior performance compared to other baselines.

\begin{table*}[t]
\caption{Minimum distance ($2\rho_X$) of 1,000 lattice points in $[0,1]^d$ for  $d=10$, $d=20$, $d=30$, $d=40$ and $d=50$.}
\centering
\begin{tabular}{|c|c|c|c|c|c|}
\hline
  & $d=10$       & $d=20$      & $d=30$     & $d=40$   & $d=50$  \\ \hline
Algorithm~\ref{alg:6}     & 0.59632  & 1.0051   & 1.3031  & 1.5482  & 1.7571  \\ \hline
Korobov  & 0.56639  & 0.90139  & 1.0695  & 1.2748  & 1.3987  \\ \hline
SCS  &    0.60224 &    1.0000 &    1.2247 &    1.4142 &    1.5811 \\ \hline
Algorithm~\ref{alg:7}  & \textbf{0.62738}   &  \textbf{1.0472}    &  \textbf{1.3620}    & \textbf{1.6175}   & \textbf{1.8401}  \\ \hline
\end{tabular}
\label{MD1000}
\end{table*}

\begin{table*}[t]
\caption{Minimum distance ($2\rho_X$) of 2,000 lattice points in $[0,1]^d$ for  $d=10$, $d=20$, $d=30$, $d=40$ and $d=50$.}
\centering
\begin{tabular}{|c|c|c|c|c|c|}
\hline
 & $d=10$       & $d=20$      & $d=30$     & $d=40$   & $d=50$   \\ \hline
 Algorithm~\ref{alg:6}     & 0.54658  & 0.95561  & 1.2595   & 1.4996  & 1.7097  \\ \hline
Korobov  & 0.51536  & 0.80039  & 0.96096  & 1.1319  & 1.2506  \\ \hline
SCS &    0.57112 &    0.98420 &    1.2247 &    1.4142 &    1.5811 \\ \hline
Algorithm~\ref{alg:7} & \textbf{0.58782}   & \textbf{1.0144}    & \textbf{1.3221}    & \textbf{1.5758 } & \textbf{1.8029 }  \\ \hline
\end{tabular}
\label{MD2000}
\end{table*}

\begin{table*}[t]
\caption{Minimum distance ($2\rho_X$) of 3,000 lattice points in $[0,1]^d$ for  $d=10$, $d=20$, $d=30$, $d=40$ and $d=50$.}
\centering
\begin{tabular}{|c|c|c|c|c|c|}
\hline
  & $d=10$       & $d=20$      & $d=30$     & $d=40$   & $d=50$   \\ \hline
Algorithm~\ref{alg:6}     & 0.53359 & 0.93051  & 1.2292   & 1.4696   & 1.7009  \\ \hline
Korobov  & 0.50000     & 0.67185  & 0.82285  & 0.95015  & 1.0623  \\ \hline
SCS &    0.52705 &    0.74536 &    0.91287 &    1.0541     & 1.1785 \\ \hline
Algorithm~\ref{alg:7} & \textbf{0.56610}  & \textbf{0.98601}   & \textbf{1.2979}    & \textbf{1.5553}    &  \textbf{1.7771}   \\ \hline
\end{tabular}
\label{MD3000}
\end{table*}

\subsection{Comparison of minimum distance generated by different methods}

 We  evaluate the proposed Algorithm~\ref{alg:6} and Algorithm~\ref{alg:7} by comparing them  with searching in  Korobov form~\citep{Korobov} and  SCS~\citep{lyu2017spherical,SCSsearch}.  We fix $M=50$ for Algorithm~\ref{alg:6} and Algorithm~\ref{alg:7} in all the experiments. The number of iterations of SCS search~\citep{lyu2017spherical,SCSsearch} is set to $T=150$, and  number of iterations of SCS search as a subroutine in Algorithm~\ref{alg:7} is set to $T=3$. 

 The  minimum distances ($2\rho_X$) of $1,000$ points, $2,000$ points and $3,000$ points generated by different methods are summarized in Tables~\ref{MD1000},~\ref{MD2000} and~\ref{MD3000}, respectively. 
Algorithm~\ref{alg:7} can achieve a larger separate (minimum) distance than other searching methods. This means that Algorithm~\ref{alg:7} can generate points set with a smaller covering radius (fill distance). Thus,  it can generate more robust initialization for BO. Moreover, Algorithm~\ref{alg:7} can also be used to generate points for integral approximation on $[0,1]^d$.

\subsection{Comparison between lattice points and random points}

\begin{figure*}[t]
\centering
\subfigure[\scriptsize{100 lattice points }]{
\label{fig2a_l}
\includegraphics[width=0.45\linewidth]{./results/lattice100.png}}
\subfigure[\scriptsize{100 random points }]{
\label{fig2c_l}
\includegraphics[width=0.45\linewidth]{./results/random100.png}}
\subfigure[\scriptsize{1000 lattice points}]{
\label{fig2f_l}
\includegraphics[width=0.45\linewidth]{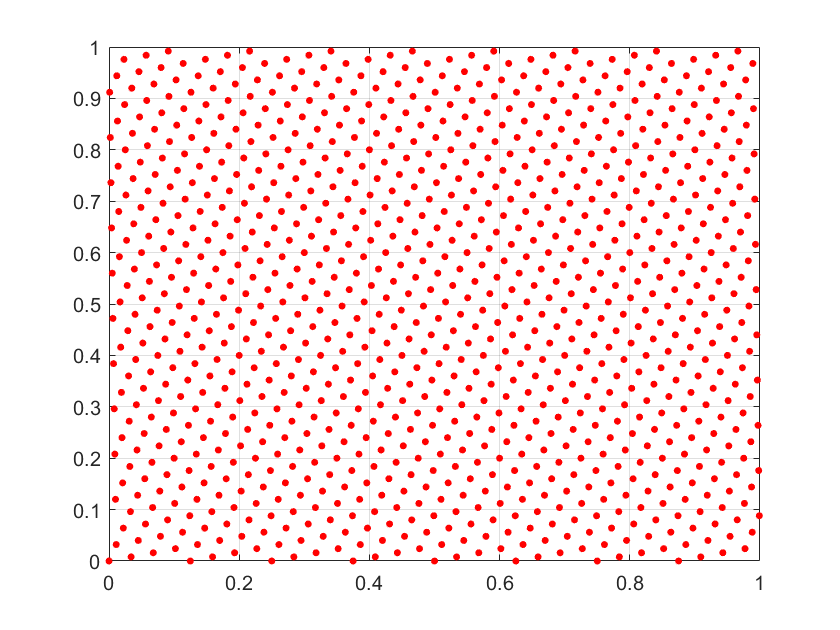}}
\subfigure[\scriptsize{1000 random points}]{
\label{fig2b_l}
\includegraphics[width=0.45\linewidth]{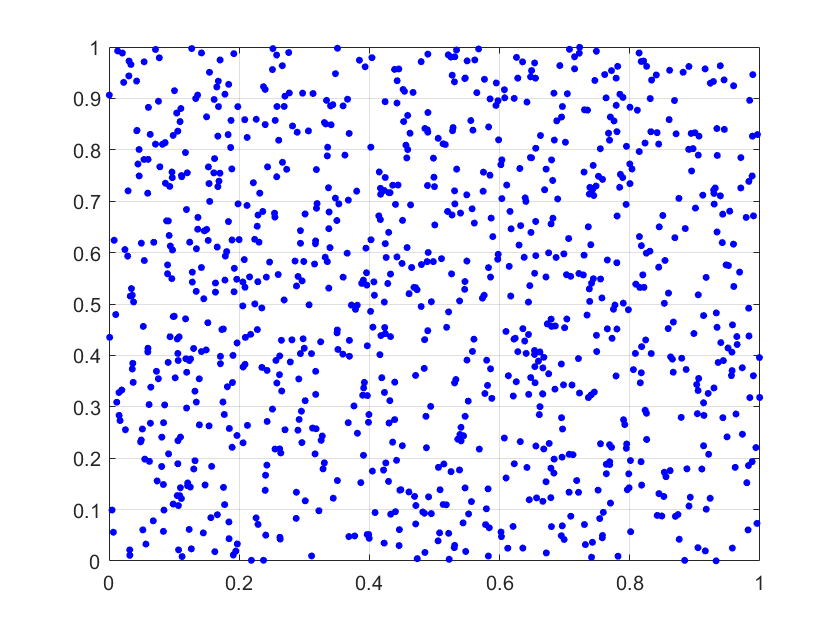}}

\caption{Lattice Points and Random Points on $[0,1]^2$}
\label{fig_l}
\end{figure*}

The points generated by Algorithm~\ref{alg:6} and uniform sampling are presented in Figure~\ref{fig_l}.
We can observe that the points generated by Algorithm~\ref{alg:6} cut the domain into several cells. It obtains a smaller covering radius (fill distance) than the random sampling. Thus, it can be used as a robust initialization of BO.



\begin{table}[h]
\centering
\caption{Test functions }
\label{testF}
\begin{tabular}{lll}
\hline
name                          & function & domain       \\ \hline
Rosenbrock                    &  $\sum\limits_{i = 1}^{d - 1} {\left( {100{{({x_{i + 1}} - x_i^2)}^2} + {{(1 - {x_i})}^2}} \right)} $        & $[-2, 2 ]^d$  \\
Nesterov &  $\frac{1}{4}\left| {{x_1} - 1} \right| + \sum\limits_{i = 1}^{d - 1} {\left| {{x_{i + 1}} - 2\left| {{x_i}} \right| + 1} \right|} $        & $[-2, 2 ]^d$   \\
Different-Powers              &  $\sum\limits_{i = 1}^d {{{\left| {{x_i}} \right|}^{2 + 10\frac{{i - 1}}{{d - 1}}}}} $        & $[-2, 2 ]^d$  \\
Dixon-Price                   &  ${\left( {{x_1} - 1} \right)^2} + \sum\limits_{i = 2}^d {i{{\left( {2x_i^2 - {x_{i - 1}}} \right)}^2}} $        & $[-2, 2 ]^d$   \\ 

Ackley                        &    $ - 20\exp ( - 0.2\sqrt {\frac{1}{d}\sum\limits_{i = 1}^d {x_i^2} } ) - \exp (\frac{1}{d}\sum\limits_{i = 1}^d {\cos (2\pi {x_i})} ) + 20 + \exp (1)$ 
& $[-2, 2 ]^d$  \\
Levy                          & $\begin{array}{l}
{\sin ^2}(\pi {w_1}) + \sum\limits_{i = 1}^{d - 1} {{{({w_i} - 1)}^2}(1 + 10{{\sin }^2}(\pi {w_i} + 1))} \\ + {({w_d} - 1)^2}(1 + {\sin ^2}(2\pi {w_d}))\\
{\rm{where}} \; {w_i} = 1 + ({x_i} - 1)/4,\;i \in \{ 1,...,d\} 
\end{array}$       & $[-10, 10 ]^d$ \\
\\ \hline
\end{tabular}
\end{table}

\section{Comparison with Bull's Non-adaptive Batch Method }

\cite{bull} presents a non-adaptive batch method with all the query points except one being fixed at the beginning.  As mentioned by Bull, this method is not practical. However,  \cite{bull} does not present an adaptive batch method.  We compare our adaptive batch method with Bull's non-adaptive method on Rosebrock and Ackley functions. The mean values of simple regret over 30 independent runs are presented in Figure~\ref{fig_bull}, which shows that Bull's non-adaptive method has a very slowly decreasing simple regret.

\begin{figure}[t]
\centering
\subfigure[\scriptsize{Rosenbrock function }]{
\label{fig_Bull_l}
\includegraphics[width=0.45\linewidth]{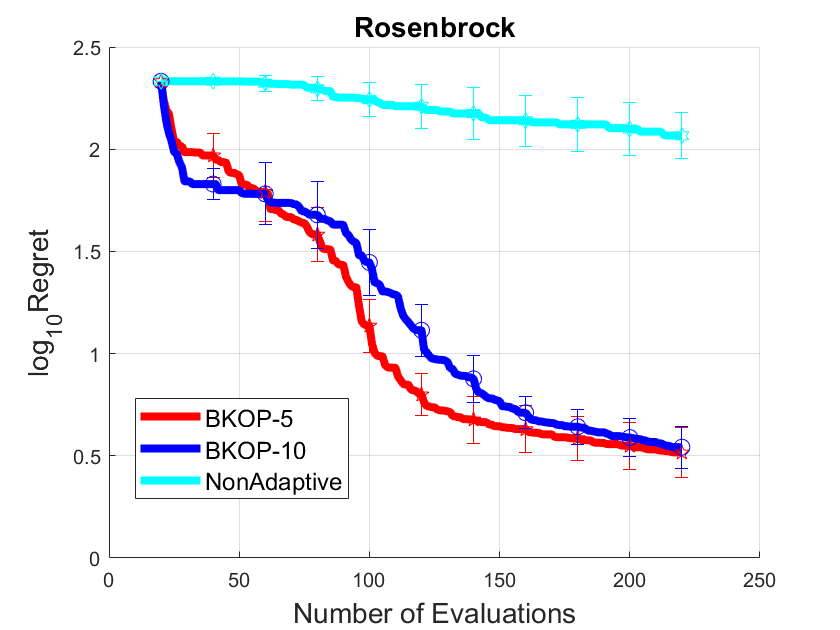}}
\subfigure[\scriptsize{Ackley function }]{
\label{fig_Bull_2}
\includegraphics[width=0.45\linewidth]{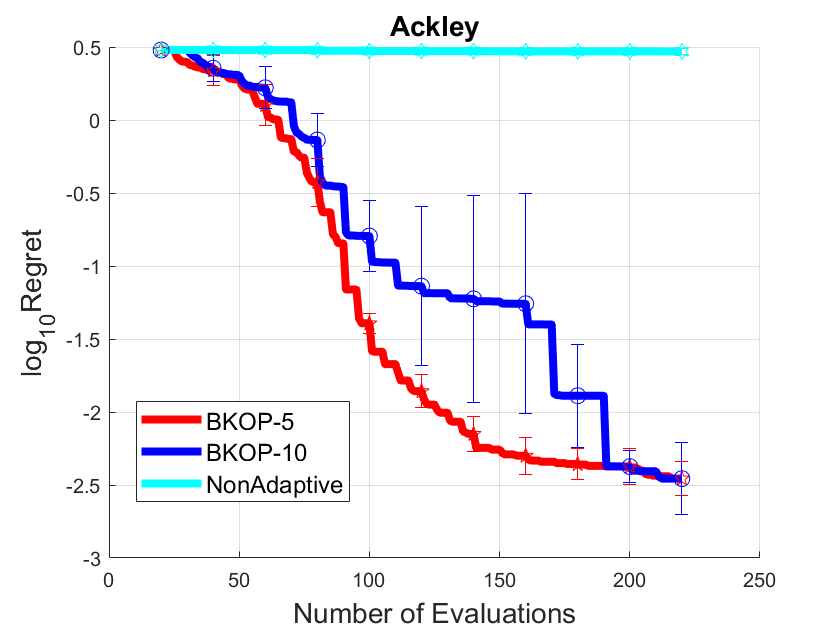}}
\caption{The mean value of simple regret over 30 runs on Rosenbrock and Ackley function }
\label{fig_bull}
\end{figure}

\section{Experiments}\label{experiments}

 In this section, we focus on the evaluation of the proposed batch method.
We evaluate the proposed Batch kernel optimization (BKOP) by comparing it with GP-BUCB \citep{GPBUCB} and GP-UCB-PE\citep{UCBPE} on several synthetic benchmark test problems,  hyperparameter tuning of a deep network on CIFAR100~\citep{cifar100}  and the robot pushing task in~\citep{mes}. An empirical study of our fast rank-1 lattice searching method is included in the supplementary material.

\begin{figure}[t]
\centering
\subfigure[\scriptsize{ Simple regret on network tuning task on CIFAR100}]{
\label{Cifar100}
\includegraphics[width=0.475\linewidth]{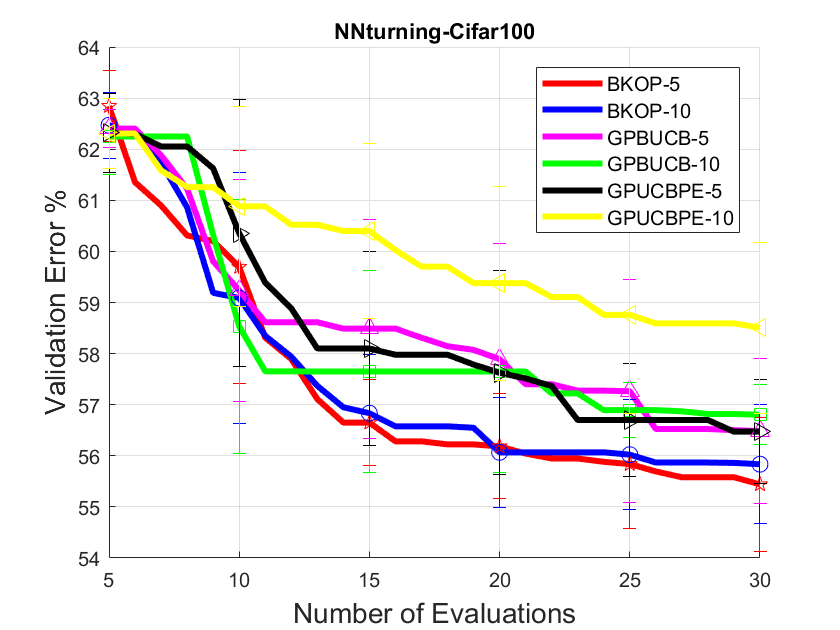}}
\subfigure[\scriptsize{ Simple regret on robot pushing task }]{
\label{robotPush}
\includegraphics[width=0.475\linewidth]{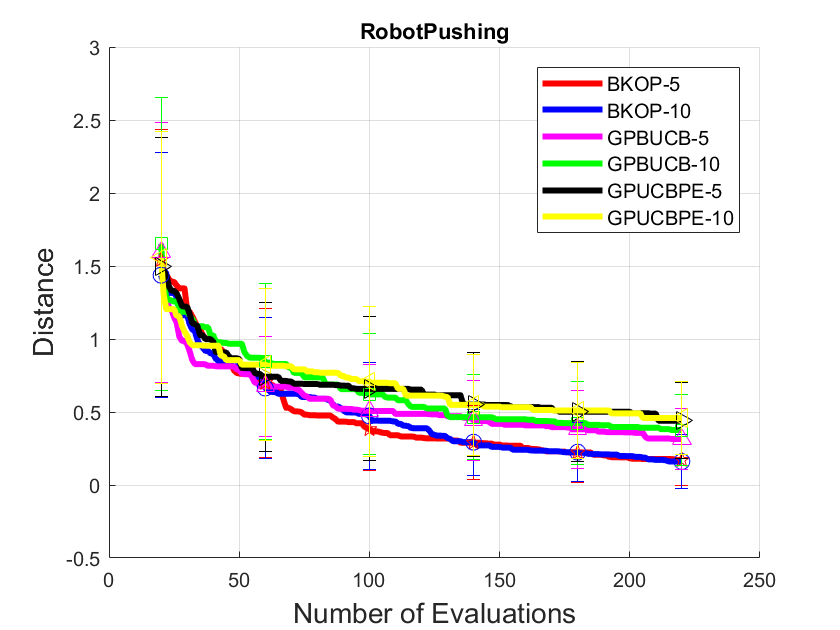}}
\caption{The mean value of simple regret on network tuning task and robot pushing task. }
\end{figure}

  \textbf{Synthetic benchmark problems:} The synthetic test functions and the domains employed are listed in Table~\ref{testF}, which includes nonconvex, nonsmooth and multimodal functions.
  
  We fix the weight of the covariance term in the acquisition function of BKOP to one in all the experiments.
  For all the synthetic test problems, we set the dimension of the domain $d=6$, and we set the batch size to $L=5$ and $L=10$ for all the batch BO algorithms.  We use the ARD Mat\'ern  5/2 kernel for all the methods.
 Instead of finding the optimum by discrete approximation,  we employ the CMA-ES algorithm~\citep{CMAES}  to optimize the acquisition function in the continuous domain $\mathcal{X}$ for all the methods, which usually improves the performance compared with discrete approximation. For each test problem, we use 20   rank-1 lattice points resized in the domain $\mathcal{X}$ as the initialization.  All the methods use the same initial points. 
 
 The mean value and error bar of the simple regret over 30 independent runs concerning different algorithms are presented in Figure~\ref{fig1}.  We can observe that BKOP with batch sizes 5 and 10 performs better than the other methods with the same batch size.  Moreover, algorithms with batch size 5 achieve faster-decreasing regret compared with batch size 10. BKOP achieves significantly low regret compared with the other methods on the Different-Powers and Rosenbrock test functions.
 
 \textbf{Hyperparameter tuning of network:} 
 We evaluate BKOP on hyperparameter tuning of the network on the CIFAR100 dataset. The network we employed contains three hidden building blocks, each one consists of one convolution layer, one batch normalization layer and one RELU layer. The depth of a building block is defined as the repeat number of these three layers.  Seven hyperparameters are used in total for searching, namely, the depth of the building block ($\{1,2,3\}$ ), the  initialized learning rate 
for SGD ($[10^{-4},10^{-1}]$), the momentum  weight ($[0.1,0.95]$), weight of L2 regularization ($[10^{-10},10^{-2}]$), and  three hyperparameters related to the  filter size for each building block, the domain of these three parameters is $\{2 \times 2,3 \times 3,4 \times 4\}$. We employ the default training set (i.e., $50,000$ samples) for training, and use the default test set (i.e., $10,000$ samples) to compute the validation error regret of automatic hyperparameter tuning for all the methods. 

We employ five rank-1 lattice points resized in the domain as the initialization. All the methods use the same initial points.
The mean value of the simple regret of the validation error in percentage over 10 independent runs is presented in Figure~\ref{Cifar100}. We can observe that  BKOP with both batch size 5 and 10 outperforms the others. Moreover, the performance of GP-UCB-PE with batch size 10 is worse than the others.

\textbf{Robot Pushing Task:}
 We further evaluate the performance of BKOP on the robot pushing task in~\citep{mes}.
 The goal of this task is to select a good action for pushing an object to a target location. The 4-dimensional robot pushing problem consists of the robot location $(x,y)$ and angle $\theta$ and the pushing duration $\tau$ as the input. And it outputs the distance between the pushed object and the target location as the function value. We employ 20  rank-1 lattice points as initialization. All the methods use the same initialization points. Thirty goal locations are randomly generated for testing.  All the methods use the same goal locations.  The mean value and error bars over 30 trials are presented in  Figure~\ref{robotPush}.
  We can observe that  BKOP with both batch size 5 and batch size 10 can achieve lower regret compared with GP-BUCB and GP-UCB-PE.

  \section{Conclusion}\label{conclusion}

We analyzed black-box optimization for functions with a bounded norm in RKHS.
 For sequential BO, we obtain a similar acquisition function to  GP-UCB, but with a constant deviation weight. For batch BO, we proposed the BKOP  algorithm, which is competitive with, or better
than, other batch confidence-bound methods on a variety of tasks.   Theoretically, we derive regret bounds for both the sequential and batch cases regardless of the choice of kernels, which are more general than the previous studies. 
Furthermore, we derive adversarial regret bounds with respect to the covering radius, which provides an important insight to design robust initialization for BO. To this end, we proposed fast searching methods to construct a good rank-1 lattice. Empirically, the proposed searching methods can obtain a large packing radius (separate distance).




\newpage


\bibliography{BO.bib}

\newpage
 
\newpage
\onecolumn

\appendix

\newpage

\section{Proof of Theorem 1}

\begin{lemma}
Suppose $f \in \mathcal{H}_k$ associated with $k(x,x) $, then  ${\left( {{m_t}(x) - f(x)} \right)^2} \le \left\| f \right\|_{{\mathcal{H}_k}}^2 \sigma _t^2(x)$
\end{lemma}

\begin{proof}
Let ${\bf{\alpha }} = {\bf{K}}_t^ - {{\bf{k}}_t}{(x)}$. Then we have 
\begin{align}
{\left( {{m_t}(x) - f(x)} \right)^2} & = {\left( {\sum\limits_{i = 1}^t {{\alpha _i}f\left( {{x_i}} \right)}  - f(x)} \right)^2}\\
& = {\left( {\left\langle {\sum\limits_{i = 1}^t {{\alpha _i}k\left( {{x_i}, \cdot } \right)}  - k(x, \cdot ),f} \right\rangle } \right)^2}\\
& \le \left\langle {f,f} \right\rangle \left\langle {\sum\limits_{i = 1}^t {{\alpha _i}k\left( {{x_i}, \cdot } \right)}  - k(x, \cdot ),\sum\limits_{i = 1}^t {{\alpha _i}k\left( {{x_i}, \cdot } \right)}  - k(x, \cdot ),} \right\rangle \\
& = \left\| f \right\|_{{\mathcal{H}_k}}^2\left\| {\sum\limits_{i = 1}^t {{\alpha _i}k\left( {{x_i}, \cdot } \right)}  - k(x, \cdot )} \right\|_{{\mathcal{H}_k}}^2  \label{H}
\end{align}

In addition, we can achieve that
\begin{align}
\left\| {\sum\limits_{i = 1}^t {{\alpha _i}k\left( {{x_i}, \cdot } \right)}  - k(x, \cdot )} \right\|_{{\mathcal{H}_k}}^2 & = k(x,x) - 2\sum\limits_{i = 1}^t {{\alpha _i}k\left( {{x_i},x} \right) + } \sum\limits_{i = 1}^t {\sum\limits_{j = 1}^t {{\alpha _i}{\alpha _j}k\left( {{x_i},{x_j}} \right)} } \\
& = k(x,x) - 2{{\bf{\alpha }}^T}{{\bf{k}}_t}(x) + {{\bf{\alpha }}^T}{{\bf{K}}_t}{\bf{\alpha }}\\
& = k(x,x) - 2{{\bf{k}}_t}{(x)^T}{{\bf{K}}_t}^ - {{\bf{k}}_t}(x) + {{\bf{k}}_t}{(x)^T}{{\bf{K}}_t}^ - {{\bf{K}}_t}{{\bf{K}}_t}^ - {{\bf{k}}_t}(x)\\
& = k(x,x) - {{\bf{k}}_t}{(x)^T}{{\bf{K}}_t}^ - {{\bf{k}}_t}(x)\\
& = \sigma _t^2(x) \label{sig}
\end{align}

Plug (\ref{sig}) into (\ref{H}), we can attain ${\left( {{m_t}(x) - f(x)} \right)^2} \le \left\| f \right\|_{{\mathcal{H}_k}}^2 \sigma _t^2(x)$.

\end{proof}

\begin{lemma}
$f(x^*) - f({x_t}) \le 2\left\| f \right\|_{{\mathcal{H}_k}}^{}\sigma _{t - 1}^{}({x_t})$.
\end{lemma}

\begin{proof}
From Lemma 1 and Algorithm 1, we can achieve that
\begin{align}
f({x^*}) - f({x_t}) & \le {m_{t - 1}}({x^*}) + {\left\| f \right\|_{{\mathcal{H}_k}}}{\sigma _{t - 1}}({x^*}) - f({x_t})\\
& \le {m_{t - 1}}({x_t}) + {\left\| f \right\|_{{\mathcal{H}_k}}}{\sigma _{t - 1}}({x_t}) - f({x_t})\\
& \le {\left\| f \right\|_{{\mathcal{H}_k}}}{\sigma _{t - 1}}({x_t}) + {\left\| f \right\|_{{\mathcal{H}_k}}}{\sigma _{t - 1}}({x_t})\\
& = 2{\left\| f \right\|_{{\mathcal{H}_k}}}{\sigma _{t - 1}}({x_t})
\end{align}
\end{proof}

\begin{lemma}
\label{LemmaVc}
Let ${\widehat \sigma _t^2}(x) = k(x,x) - {{\bf{k}}_t}{(x)^T}{({\sigma ^2}I + {{\bf{K}}_t})^ - }{{\bf{k}}_t}(x)$. Then ${\sigma _t^2}(x) \le {\widehat \sigma _t^2}(x)$.
\end{lemma}

\begin{proof}
Since kernel matrix ${\bf{K}}_t$ is positive semi-definite, it follows that ${\bf{K}}_t = U^T \Lambda U $, where $U$ is orthonormal matrix consists of eigenvectors, $\Lambda$ is a diagonal matrix consists of eigenvalues.

Let ${\bf{\beta }} = U{{\bf{k}}_t}(x)$, then we can achieve that 
\begin{align}
{{\bf{k}}_t}{(x)^T}{({\sigma ^2}I + {{\bf{K}}_t})^ - }{{\bf{k}}_t}(x) & = \sum\limits_{i = 1}^t {\frac{{\beta _i^2}}{{{\sigma ^2} + {\lambda _i}}}} \\
& \le \sum\limits_{i = 1}^t {\frac{{\beta _i^2}}{{{\lambda _i}}}}  = {{\bf{\beta }}^T}{\Lambda ^ - }{\bf{\beta }} \\ & = {{\bf{k}}_t}{(x)^T}{U^T}{\Lambda ^ - }U{{\bf{k}}_t}(x)\\
& = {{\bf{k}}_t}{(x)^T}{\bf{K}}_t^ - {{\bf{k}}_t}(x)
\end{align}
It follows that 
\begin{align}
\sigma _t^2(x) & = k(x,x) - {{\bf{k}}_t}{(x)^T}{\bf{K}}_t^ - {{\bf{k}}_t}(x)\\
 & \le k(x,x) - {{\bf{k}}_t}{(x)^T}{({\sigma ^2}I + {{\bf{K}}_t})^ - }{{\bf{k}}_t}(x)\\
& = {\widehat \sigma _t^2}(x)
\end{align}

\end{proof}

Now, we are ready to prove Theorem 1. 
\begin{proof}
First, we have 
\begin{align}
{R_T} & = \sum\limits_{i = 1}^T {f({x^*}) - f({x_t})} \\
& \le 2{\left\| f \right\|_{{\mathcal{H}_k}}}\sum\limits_{i = 1}^T {{\sigma _{t - 1}}({x_t})} \\
& \le 2{\left\| f \right\|_{{\mathcal{H}_k}}}\sqrt {T\sum\limits_{i = 1}^T {\sigma _{t - 1}^2({x_t})} } \label{R}
\end{align}

Since $s \le \frac{{{1}}}{{\log \left( {1 + {\sigma ^{ - 2}}} \right)}}\log \left( {1 + \sigma ^{-2} s} \right)$ for $s \in \left[ {0,1} \right]$ and $ 0 \le \widehat \sigma _{t - 1}^2({x_t}) \le k(x,x) \le 1 $ for all $t \ge 1$, it follows that 
\begin{align}
\sum\limits_{i = 1}^T {\sigma _{t - 1}^2({x_t})} \le \sum\limits_{i = 1}^T {\widehat \sigma _{t - 1}^2({x_t})}  & \le \frac{1}{{\log (1 + {\sigma ^{ - 2}})}}\sum\limits_{i = 1}^T {\log (1 + {\sigma ^{ - 2}}\widehat \sigma _{t - 1}^2({x_t}))} \\
& \le \frac{{2{\gamma _T}}}{{\log (1 + {\sigma ^{ - 2}})}} \label{MI}
\end{align}

Together (\ref{R}) and (\ref{MI}), we can attain that 
\begin{align}
{R_T} & \le 2{\left\| f \right\|_{{H_k}}}\sqrt {T\frac{{2{\gamma _T}}}{{\log (1 + {\sigma ^{ - 2}})}}} \\
& = {\left\| f \right\|_{{\mathcal{H}_k}}}\sqrt {T{C_1}{\gamma _T}} 
\end{align}

It follows that $r_T \le \frac{{{R_T}}}{T} \le {\left\| f \right\|_{{\mathcal{H}_k}}}\sqrt {\frac{{C_1{\gamma _T}}}{T}}  $.

\end{proof}


\section{Proof of Theorem 2}

\begin{lemma}
\label{RKHSbound}
Suppose $f \in \mathcal{H}_k$ associated with kernel $k(x,x)$, then ${\left( {\sum\nolimits_{i = 1}^L {{m_t}({{\widehat x_i}}) - \sum\nolimits_{i = 1}^L {f({{\widehat x_i}})} } } \right)^2} \le \left\| f \right\|_{{\mathcal{H}_k}}^2  ({{\bf{1}}^T}{\bf{A1}})$, where $\bf{A}$ denotes the kernel matrix (covariance matrix)  with ${{\bf{A}}_{ij}} = k({\widehat x_i},{\widehat x_j}) - {{\bf{k}}_t}{({\widehat x_i})^T}{\bf{K}}_t^ - {{\bf{k}}_t}({\widehat x_j})$. 
\end{lemma}

\begin{proof}
Let ${\bf{\alpha }}^i = {{\bf{k}}_t}{({\widehat x_i})^T}{\bf{K}}_t^ - $. Then we have 
\begin{align}
{{{\left( {\sum\limits_{i = 1}^L {{m_t}({{\widehat x}_i})}  - \sum\limits_{i = 1}^L {f({{\widehat x}_i})} } \right)}^2}}&  = {{\left( {\sum\limits_{i = 1}^L {\sum\limits_{l = 1}^t {{\alpha _l^i}f\left( {{{ x}_l}} \right)}  - \sum\limits_{i = 1}^L {f({{\widehat x}_i})} } } \right)}^2} \\
& = {{\left( {\left\langle {\sum\limits_{i = 1}^L {\sum\limits_{l = 1}^t {\alpha _l^ik\left( {{x_l}, \cdot } \right)} }  - \sum\limits_{i = 1}^L {k({{\widehat x}_i}, \cdot )} ,f} \right\rangle } \right)}^2}\\
& \le \left\| f \right\|_{{{\cal H}_k}}^2\left\| {\sum\limits_{i = 1}^L {\sum\limits_{l = 1}^t {\alpha _l^ik\left( {{x_l}, \cdot } \right)} }  - \sum\limits_{i = 1}^L {k({{\widehat x}_i}, \cdot )} } \right\|_{{{\cal H}_k}}^2
\end{align}
In addition, we have
\begin{align}
\nonumber
& \left\| {\sum\limits_{i = 1}^L {\sum\limits_{l = 1}^t {\alpha _l^ik\left( {{x_l}, \cdot } \right)} }  - \sum\limits_{i = 1}^L {k({{\widehat x}_i}, \cdot )} } \right\|_{{{\cal H}_k}}^2 \\ & = \sum\limits_{i = 1}^L {\sum\limits_{j = 1}^L {k({{\widehat x}_i},{{\widehat x}_j})} }  - 2\sum\limits_{i = 1}^L {\sum\limits_{j = 1}^L {\sum\limits_{l = 1}^t {\alpha _l^ik\left( {{x_l},{{\widehat x}_j}} \right)} } }  + \sum\limits_{i = 1}^L {\sum\limits_{j = 1}^L {\sum\limits_{n = 1}^t {\sum\limits_{l = 1}^t {\alpha _l^i\alpha _n^jk\left( {{x_l},{x_n}} \right)} } } } \\
& = \sum\limits_{i = 1}^L {\sum\limits_{j = 1}^L {k({{\widehat x}_i},{{\widehat x}_j})} }  - 2\sum\limits_{i = 1}^L {\sum\limits_{j = 1}^L {{{\bf{k}}_t}{{({{\widehat x}_i})}^T}{\bf{K}}_t^ - {{\bf{k}}_t}({{\widehat x}_j})} }  + \sum\limits_{i = 1}^L {\sum\limits_{j = 1}^L {{{\bf{k}}_t}{{({{\widehat x}_i})}^T}{\bf{K}}_t^ - {{\bf{k}}_t}({{\widehat x}_j})} } \\
& = \sum\limits_{i = 1}^L {\sum\limits_{j = 1}^L {{{\bf{A}}_{ij}} = {{\bf{1}}^T}{\bf{A1}}} } 
\end{align}
Thus, we obtain ${\left( {\sum\nolimits_{i = 1}^L {{m_t}({{\widehat x_i}}) - \sum\nolimits_{i = 1}^L {f({{\widehat x_i}})} } } \right)^2} \le \left\| f \right\|_{{\mathcal{H}_k}}^2  ({{\bf{1}}^T}{\bf{A1}})$.

\end{proof}

\begin{lemma}
Suppose $f \in \mathcal{H}_k$ associated with kernel $k(x,x)$, then $\frac{1}{L}\sum\limits_{i = 1}^L {\left( {f({x^*}) - f({x_{(n - 1)L + i}})} \right)}  \le 2\left\| f \right\|_{{\mathcal{H}_k}}  \sqrt { \frac{tr\left( {{{\rm cov}_{n - 1}}({X_n},{X_n})} \right)}{L}} $, where covariance matrix ${{\mathop{\rm cov}}_{{n - 1}}({X_n},{X_n})}$ constructed as Eq.(\ref{BnoiseFree}) and $X_n = \{x_{(n - 1)L + 1},...,x_{nL}\}$.
\end{lemma}

\begin{proof}

Let $X^*=\{x^*,...,x^*\}$ be $L$ copies of $x^*$. Then, we obtain that

\begin{align}
& \frac{1}{L}\sum\limits_{i = 1}^L {\left( {f({x^*}) - f({x_{(n - 1)L + i}})} \right)}  = f({x^*}) - \frac{1}{L}\sum\limits_{i = 1}^L {f({x_{(n - 1)L + i}})} \\
& \le {m_{(n - 1)L}}({x^*}) + {\left\| f \right\|_{{{\cal H}_k}}}{\sigma _{(n - 1)L}}({x^*}) - \frac{1}{L}\sum\limits_{i = 1}^L {f({x_{(n - 1)L + i}})} \\
& = \frac{1}{L}\sum\limits_{i = 1}^L {{m_{(n - 1)L}}({x^*})}  + {\left\| f \right\|_{{{\cal H}_k}}}\left( {2\sqrt {\frac{{tr\left( {{{\rm cov}_{n - 1}}({X^*},{X^*})} \right)}}{L}}  - \sqrt {\frac{{{{\bf{1}}^T}{{\rm cov}_{n - 1}}({X^*},{X^*}){\bf{1}}}}{{{L^2}}}} } \right)  \nonumber \\ & \;\;\;\; -  \frac{1}{L}\sum\limits_{i = 1}^L {f({x_{(n - 1)L + i}})} \\
& \le \frac{1}{L}\sum\limits_{i = 1}^L {{m_{(n - 1)L}}({x_{(n - 1)L + i}})}  + {\left\| f \right\|_{{{\cal H}_k}}}\left( {2\sqrt {\frac{{tr\left( {{{\rm cov}_{n - 1}}({X_n},{X_n})} \right)}}{L}}  - \sqrt {\frac{{{{\bf{1}}^T}{{\rm cov}_{n - 1}}({X_n},{X_n}){\bf{1}}}}{{{L^2}}}} } \right)  \nonumber \\ & \;\;\;\; -  \frac{1}{L}\sum\limits_{i = 1}^L {f({x_{(n - 1)L + i}})} \\
& \le {\left\| f \right\|_{{{\cal H}_k}}} \! \left( \! {2\sqrt {\frac{{tr\left( {{{\rm cov}_{n - 1}}({X_n},{X_n})} \right)}}{L}} \! - \!\!\sqrt {\frac{{{{\bf{1}}^T}{{\rm cov}_{n - 1}}({X_n},{X_n}){\bf{1}}}}{{{L^2}}}} } \right) \!\!+\! {\left\| f \right\|_{{{\cal H}_k}}} \! \sqrt {\frac{{{{\bf{1}}^T}{{\rm cov}_{n - 1}}({X_n},{X_n}){\bf{1}}}}{{{L^2}}}} \\
& = 2{\left\| f \right\|_{{{\cal H}_k}}}\sqrt {\frac{{tr\left( {{{\rm cov}_{n - 1}}({X_n},{X_n})} \right)}}{L}} 
\end{align}

\end{proof}

\begin{lemma}
Let $B_{n}$ and $A_n$ be the covariance matrix constructed by Eq.(\ref{BnoiseFree}) and Eq.(\ref{NoiseBatchV}), respectively. Then $tr(B_n) \le tr(A_n)$
\end{lemma}
\begin{proof}
It follows directly from Lemma \ref{LemmaVc}.
\end{proof}

\begin{lemma}
Let matrix ${A_{n - 1}} = {{\mathop{\rm cov}} _{n - 1}}\left( {{X_n},{X_n}} \right)$ as Eq.(\ref{NoiseBatchV}). Denote the spectral norm of matrix $A_{n - 1}$ as  $\beta _{n-1} =\left\| {A_{n - 1}}  \right\|_2 $.   Then 
${\rm{tr}}\left( {{A_{n - 1}}} \right) \le \frac{\beta _{n-1}}{{\log \left( {1 + \beta _{n-1}{\sigma ^{ - 2}}} \right)}}\log \det \left( {I + {\sigma ^{ - 2}}{A_{n - 1}}} \right)$ for any $\sigma \neq 0$.
\end{lemma}

\begin{proof}

Since $A_{n-1}$  is a positive semidefinite matrix, we can attain that the eigenvalues of    $A_{n-1}$ are all nonnegative. Without loss of generality, assume eigenvalues of    $A_{n-1}$ as $0 \le {\lambda _L} \le ... \le {\lambda _1}$. By the definition of the spectral norm $\beta _{n-1} =\left\| {A_{n - 1}}  \right\|_2 $,  we obtain that   $0 \le {\lambda _L} \le ... \le {\lambda _1} \le \beta _{n-1}$

Since $ s \le \frac{\beta _{n-1}}{{\log \left( {1 + \beta _{n-1}{\sigma ^{ - 2}}} \right)}}\log \left( {1 + {\sigma ^{ - 2}}s} \right)$ for $s \in \left[ {0,\beta _{n-1}} \right]$ and $ 0 \le {\lambda _i} \le \beta _{n-1}$, $i \in \{1,...,L \}$, we can obtain that inequality~(\ref{Leq24}) holds true for all $i \in \{1,...,L \}$
\begin{equation}
\label{Leq24}
\begin{array}{l}
{\lambda _i} \le \frac{\beta _{n-1}}{{\log \left( {1 + \beta _{n-1}{\sigma ^{ - 2}}} \right)}} \log \left( {1 + {\sigma ^{ - 2}}{\lambda _i}} \right)
\end{array}
\end{equation}

Because $\log \det \left( {I + {\sigma ^{ - 2}}{A_{n - 1}}} \right) = \sum\limits_{i = 1}^L {\log \left( {1 + {\sigma ^{ - 2}}{\lambda _i}} \right)} $, we can achieve that
\begin{equation}
\label{eq25}
\begin{array}{l}
{\rm{tr}}\left( {{A_{n - 1}}} \right) = \sum\limits_{i = 1}^L {{\lambda _i}}  \le \frac{\beta _{n-1}}{{\log \left( {1 + \beta _{n-1}{\sigma ^{ - 2}}} \right)}} \log \det \left( {I + {\sigma ^{ - 2}}{A_{n - 1}}} \right)
\end{array}
\end{equation}
\end{proof}

\begin{lemma}
Let $T=NL$,  ${{\bf{K}}_T}$ be the $T\times T$ sized  kernel matrix and  ${{\bf{I}}_L}$ be the $L \times L$ sized idendity matrix.   Then
 $ \frac{1}{2}\log \det \left( {I + {\sigma ^{ - 2}}{{\bf{K}}_T}} \right) = \frac{1}{2}\sum\limits_{n = 1}^N {\log \det \left( {{{\bf{I}}_L} + {\sigma ^{ - 2}}{A_{n - 1}}} \right)}  $, where matrix ${A_{n - 1}} = { \widehat{ \mathop{\rm cov}} _{n - 1}}\left( {{X_n},{X_n}} \right)$ as Eq.(\ref{NoiseBatchV}).
\end{lemma}

\begin{proof}
\begin{align}
    \frac{1}{2}\log \det \left( {{{\bf{I}}_T} + {\sigma ^{ - 2}}{{\bf{K}}_T}} \right) = \frac{1}{2}\log \det \left( {{\sigma ^2}{{\bf{I}}_T} + {{\bf{K}}_T}} \right) - \frac{1}{2}\log \det \left( {{\sigma ^2}{{\bf{I}}_T}} \right)
\end{align}

Using the determinant equation $\det \left( {\begin{array}{*{20}{c}}
A&B\\
C&D
\end{array}} \right) = \det \left( A \right) \cdot \det \left( {D - C{A^{ - 1}}B} \right)$ in linear algebra,  set $A = {\sigma ^2}{{\bf{I}}_{(N - 1)L}} + {{\bf{K}}}\left( {{{\overline X }_{N - 1}},{{\overline X }_{N - 1}}} \right)$, $B = {\bf{K}}\left( {{{\overline X }_{N - 1}},{X_N}} \right)$, $C=B^T$ and $ D = {\sigma ^2}{{\bf{I}}_L} + {\bf{K}}\left( {{X_N},{X_N}} \right)$, where ${\overline X _{N - 1}} = \{ {x_1},...,{x_{(N - 1)L}}\} $ denote all previous $N-1$  batch of points, ${X_N} = \{ x_{(N-1)L+1},...,x_{NL}\} $ denote the $N^{th}$ batch of points and $\bf{K}(\cdot,\cdot)$ denote the kernel matrix constructed by its input. Then, we can achieve that
\begin{align}
  &  \frac{1}{2}\log \det \left( {{\sigma ^2}{{\bf{I}}_T} + {{\bf{K}}_T}} \right) - \frac{1}{2}\log \det \left( {{\sigma ^2}{{\bf{I}}_T}} \right)\\ \nonumber
 & = \frac{1}{2}\log \det \left( {{\sigma ^2}{{\bf{I}}_{(N - 1)L}} +{{\bf{K}}}\left( {{{\overline X }_{N - 1}},{{\overline X }_{N - 1}}} \right)} \right) + \frac{1}{2}\log \det \left( {{\sigma ^2}{{\bf{I}}_L} + {A_{N - 1}}} \right) - \frac{1}{2}\log \det \left( {{\sigma ^2}{{\bf{I}}_T}} \right)\\ \nonumber
 & =\frac{1}{2}\log \det \left( {{\sigma ^2}{{\bf{I}}_{(N - 1)L}} + {{\bf{K}}}\left( {{{\overline X }_{N - 1}},{{\overline X }_{N - 1}}} \right)} \right) + \frac{1}{2}\log \det \left( {{{\bf{I}}_L} + {\sigma ^{ - 2}}{A_{N - 1}}} \right) - \frac{1}{2}\log \det \left( {{\sigma ^2}{{\bf{I}}_{(N - 1)L}}} \right) \nonumber
\end{align}
where ${A_{N - 1}} = {{\mathop{\rm cov}} _{N - 1}}\left( {{X_N},{X_N}} \right)$ is the  covariance matrix between $X_N$ and $X_N$ constructed as Eq.(\ref{NoiseBatchV}).

By induction, we can achieve  $\frac{1}{2}\log \det \left( {{{\bf{I}}_T} + {\sigma ^{ - 2}}{{\bf{K}}_T}} \right)  = \frac{1}{2}\sum\limits_{n = 1}^N {\log \det \left( {{{\bf{I}}_L} + {\sigma ^{ - 2}}{A_{n - 1}}} \right)} $

\end{proof}

Finally, we are ready to attain Theorem 2.
\begin{proof}
Let covariance matrix ${{\bf{A}}_{n - 1}} $   and  ${{\bf{B}}_{n - 1}} $ be constructed as Eq.(\ref{NoiseBatchV}) and Eq.~(\ref{BnoiseFree}), respectively.  Let $\beta_{n-1}= \left\| {{\bf{A}}_{n - 1}} \right\|_2$. Then, we can achieve that
\begin{align}
    {R_T} & = \sum\limits_{t = 1}^T {f({x^*}) - f({x_t})} \\
& \le 2{\left\| f \right\|_{{\mathcal{H}_k}}}\sum\limits_{n = 1}^N \sqrt {L \;tr\left( {{{\bf{B}}_{n-1}}} \right)}  \\
& \le 2{\left\| f \right\|_{{\mathcal{H}_k}}}\sum\limits_{n = 1}^N \sqrt {L \; tr\left( {{{\bf{A}}_{n-1}}} \right)}  \\
&  \le 2{\left\| f \right\|_{{{\cal H}_k}}}\sqrt {NL\sum\limits_{n = 1}^N {\;tr\left( {{{\bf{A}}_{n - 1}}} \right)} }   \\
& \le 2{\left\| f \right\|_{{{\cal H}_k}}}\sqrt {T\sum\limits_{n = 1}^N {\;\frac{{{\beta _{n - 1}}}}{{\log \left( {1 + {\beta _{n - 1}}{\sigma ^{ - 2}}} \right)}}\log \det \left( {I + {\sigma ^{ - 2}}{{\bf{A}}_{n - 1}}} \right)} } \\
& \le {\left\| f \right\|_{{{\cal H}_k}}}\sqrt {T{C_2}\sum\limits_{n = 1}^N {\;\log \det \left( {I + {\sigma ^{ - 2}}{{\bf{A}}_{n - 1}}} \right)} } \\
& \le {\left\| f \right\|_{{{\cal H}_k}}}\sqrt {T{C_2}{\gamma _T}} 
\end{align}

It follows that $r_T \le \frac{{{R_T}}}{T} \le {\left\| f \right\|_{{\mathcal{H}_k}}}\sqrt {\frac{{C_2{\gamma _T}}}{T}} $

\end{proof}

\section{Proof of Theorem 3}
\begin{lemma}
\label{Nerror}
Suppose $h= f+g \in \mathcal{H}_k^\sigma$ associated with kernel  $k^\sigma(x,y) = k(x,y)+ \sigma^2\delta(x,y) $. Suppose $f \in \mathcal{H}_k$ associated with $k$ and $g \in \mathcal{H}_{\sigma^2 \delta}$ associated with kernel $\sigma^2 \delta $.  Then for $x \ne x_i, i \in \{1,...,t\} $, we have $\left| {{\widehat m_t}(x) - f(x)} \right|  \le \left\| h \right\|_{{{\cal H}_{{k^\sigma }}}}^{}{\widehat \sigma _t}(x) + \left( {\left\| h \right\|_{{{\cal H}_{{k^\sigma }}}}^{} + \left\| g \right\|_{{{\cal H}_{{\sigma ^2}\delta }}}^{}} \right){\sigma ^{}}$.
\end{lemma}

\begin{proof}
Let ${\bf{\alpha }} = ({\bf{K}}_t + \sigma^2I) ^ {-1} {{\bf{k}}_t}{(x)}$. Then we have 
\begin{align}
{\left( {{\widehat m_t}(x) - h(x)} \right)^2} & = {\left( {\sum\limits_{i = 1}^t {{\alpha _i}h\left( {{x_i}} \right)}  - h(x)} \right)^2}\\
& = {\left( {\left\langle {\sum\limits_{i = 1}^t {{\alpha _i}k^\sigma\left( {{x_i}, \cdot } \right)}  - k^\sigma(x, \cdot ),h} \right\rangle } \right)^2}\\
& \le \left\| h \right\|_{{\mathcal{H}_{k^\sigma}}}^2\left\| {\sum\limits_{i = 1}^t {{\alpha _i}k^\sigma\left( {{x_i}, \cdot } \right)}  - k^\sigma(x, \cdot )} \right\|_{{\mathcal{H}_{k^\sigma}}}^2  \label{HN}
\end{align}

In addition, we can achieve that
\begin{align}
\nonumber
\left\| {\sum\limits_{i = 1}^t {{\alpha _i}k^\sigma\left( {{x_i}, \cdot } \right)}  - k^\sigma(x, \cdot )} \right\|_{{\mathcal{H}_{k^\sigma}}}^2 & = k^\sigma(x,x) - 2\sum\limits_{i = 1}^t {{\alpha _i}k^\sigma\left( {{x_i},x} \right) + } \sum\limits_{i = 1}^t {\sum\limits_{j = 1}^t {{\alpha _i}{\alpha _j}k^\sigma\left( {{x_i},{x_j}} \right)} } \\ 
& = k(x,x) + \sigma^2 - 2{{\bf{\alpha }}^T}{{\bf{k}}_t}(x) + {{\bf{\alpha }}^T}({{\bf{K}}_t}+\sigma^2I){\bf{\alpha }}\\
& = k(x,x) + \sigma^2 - {{\bf{k}}_t}{(x)^T}{({\bf{K}}_t+ \sigma^2I)}^ {-1} {{\bf{k}}_t}(x)\\
& = \widehat \sigma _t^2(x) + \sigma^2 \label{sigN}
\end{align}

Plug (\ref{sigN}) into (\ref{HN}), we can obtain  ${\left( {{\widehat m_t}(x) - h(x)} \right)^2} \le \left\| h \right\|_{{\mathcal{H}_{k^\sigma}}}^2(\widehat \sigma _t^2(x)+ \sigma^2)$. Thus, we achieve that
\begin{align}
    \left| {{\widehat m_t}(x) - f(x)} \right| & \le \left| {{\widehat m_t}(x) - h(x)} \right| + \left| {g(x)} \right| \\
    & \le \left\| h \right\|_{{\mathcal{H}_{k^\sigma}}}     \sqrt{ \widehat \sigma _t^2(x)+ \sigma^2} + \left\| g \right\|_{{\mathcal{H}_{\sigma^2\delta}}} \sigma \\
    & \le \left\| h \right\|_{{{\cal H}_{{k^\sigma }}}}^{}{\widehat \sigma _t}(x) + \left( {\left\| h \right\|_{{{\cal H}_{{k^\sigma }}}}^{} + \left\| g \right\|_{{{\cal H}_{{\sigma ^2}\delta }}}^{}} \right){\sigma ^{}}
\end{align}

\end{proof}

\begin{lemma}
Under same condition as Lemma \ref{Nerror}, we have
$f(x^*) - f({x_t}) \le 2\left\| h \right\|_{{{\cal H}_{{k^\sigma }}}}^{}{\widehat \sigma _{t-1}}(x_t) + 2\left( {\left\| h \right\|_{{{\cal H}_{{k^\sigma }}}}^{} + \left\| g \right\|_{{{\cal H}_{{\sigma ^2}\delta }}}^{}} \right){\sigma ^{}}$.
\end{lemma}

\begin{proof}
From Lemma \ref{Nerror} and Algorithm \ref{alg:3}, we can achieve that
\begin{align}
f({x^*}) - f({x_t})  & \le {\widehat m_{t - 1}}({x^*}) + \left\| h \right\|_{{{\cal H}_{{k^\sigma }}}}^{}{\widehat \sigma _{t-1}}(x^*) + \left( {\left\| h \right\|_{{{\cal H}_{{k^\sigma }}}}^{} + \left\| g \right\|_{{{\cal H}_{{\sigma ^2}\delta }}}^{}} \right){\sigma ^{}} - f({x_t})\\
& \le {\widehat m_{t - 1}}({x_t}) + \left\| h \right\|_{{{\cal H}_{{k^\sigma }}}}^{}{\widehat \sigma _{t-1}}(x_t) + \left( {\left\| h \right\|_{{{\cal H}_{{k^\sigma }}}}^{} + \left\| g \right\|_{{{\cal H}_{{\sigma ^2}\delta }}}^{}} \right){\sigma ^{}} - f({x_t})\\
& \le 2\left\| h \right\|_{{{\cal H}_{{k^\sigma }}}}^{}{\widehat \sigma _{t-1}}(x_t) + 2\left( {\left\| h \right\|_{{{\cal H}_{{k^\sigma }}}}^{} + \left\| g \right\|_{{{\cal H}_{{\sigma ^2}\delta }}}^{}} \right){\sigma ^{}}
\end{align}
\end{proof}

Finally, we are ready to prove Theorem \ref{TheoremNS}. 
\begin{proof}
First, we have 
\begin{align}
{R_T} & = \sum\limits_{i = 1}^T {f({x^*}) - f({x_t})} \\
& \le 2{\left\| h \right\|_{{\mathcal{H}_{k^\sigma}}}}\sum\limits_{i = 1}^T {{\widehat \sigma _{t - 1}}({x_t})} + 2T\left( {\left\| h \right\|_{{{\cal H}_{{k^\sigma }}}}^{} + \left\| g \right\|_{{{\cal H}_{{\sigma ^2}\delta }}}^{}} \right){\sigma ^{}}  \\
& \le 2{\left\| h \right\|_{{\mathcal{H}_{k^\sigma}}}}\sqrt {T\sum\limits_{i = 1}^T {\widehat \sigma _{t - 1}^2({x_t})} } + 2T\left( {\left\| h \right\|_{{{\cal H}_{{k^\sigma }}}}^{} + \left\| g \right\|_{{{\cal H}_{{\sigma ^2}\delta }}}^{}} \right){\sigma ^{}} \label{RN}
\end{align}

Since $s \le \frac{{{B}}}{{\log \left( {1 + B{\sigma ^{ - 2}}} \right)}}\log \left( {1 + \sigma ^{-2} s} \right)$ for $s \in \left[ {0,B} \right]$ and $ 0 \le \widehat \sigma _{t - 1}^2({x_t}) \le k^\sigma(x,x) \le B $ for all $t \ge 1$, it follows that 
\begin{align}
 \sum\limits_{i = 1}^T {\widehat \sigma _{t - 1}^2({x_t})}  & \le \frac{B}{{\log (1 + B{\sigma ^{ - 2}})}}\sum\limits_{i = 1}^T {\log (1 + {\sigma ^{ - 2}}\widehat \sigma _{t - 1}^2({x_t}))} \\
& \le \frac{{2B{\gamma _T}}}{{\log (1 + B{\sigma ^{ - 2}})}} \label{MIN}
\end{align}

Together (\ref{RN}) and (\ref{MIN}), we can attain that 
\begin{align}
{R_T} & \le 2{\left\| h \right\|_{{\mathcal{H}_{k^\sigma}}}}\sqrt {T\frac{{2B{\gamma _T}}}{{\log (1 + B{\sigma ^{ - 2}})}}} +  2T\left( {\left\| h \right\|_{{{\cal H}_{{k^\sigma }}}}^{} + \left\| g \right\|_{{{\cal H}_{{\sigma ^2}\delta }}}^{}} \right){\sigma ^{}}  \\
& = {\left\| h \right\|_{{\mathcal{H}_{k^\sigma}}}}\sqrt {T{C_3}{\gamma _T}}   + 2T\left( {\left\| h \right\|_{{{\cal H}_{{k^\sigma }}}}^{} + \left\| g \right\|_{{{\cal H}_{{\sigma ^2}\delta }}}^{}} \right){\sigma ^{}} 
\end{align}

It follows that $r_T \le \frac{{{R_T}}}{T} \le  {\left\| h \right\|_{{\mathcal{H}_{k^\sigma}}}}\sqrt {\frac{{C_3{\gamma _T}}}{T}}   + 2\left( {\left\| h \right\|_{{{\cal H}_{{k^\sigma }}}}^{} + \left\| g \right\|_{{{\cal H}_{{\sigma ^2}\delta }}}^{}} \right){\sigma ^{}} $.

\end{proof}

\section{Proof of Theorem \ref{batchBOnoise}}

\begin{lemma}
\label{RKHSboundnoisy}
Suppose $h= f+g \in \mathcal{H}_k^\sigma$ associated with kernel  $k^\sigma(x,y) = k(x,y)+ \sigma^2\delta(x,y) $. Suppose $f \in \mathcal{H}_k$ associated with $k$ and $g \in \mathcal{H}_{\sigma^2 \delta}$ associated with kernel $\sigma^2 \delta $.   Suppose  $\widehat x_i \ne x_j, i \in \{1,...,L\}, j \in \{1,...,t\}$,  then we have 
\begin{align}
  \left| {\sum\nolimits_{i = 1}^L {{m_t}({{\hat x}_i}) - \sum\nolimits_{i = 1}^L {f({{\hat x}_i})} } } \right| \le \left\| h \right\|_{{{\cal H}_{{k^\sigma }}}}^{}\sqrt {{{\bf{1}}^T}{\bf{A1}} + L^2{\sigma ^2}} + L \left\| g \right\|_{{{\cal H}_{{\sigma^2\delta }}}} \sigma
\end{align}
 where $\bf{A}$ denotes the kernel covariance matrix   with ${{\bf{A}}_{ij}} = k({\widehat x_i},{\widehat x_j}) - {{\bf{k}}_t}{({\widehat x_i})^T} ({\bf{K}}_t + \sigma^2I) ^ {-1} {{\bf{k}}_t}({\widehat x_j})$ 
\end{lemma}

\textbf{Remark:} Further require $\widehat x_i \ne \widehat x_j,  \forall i,j \in \{1,...,L\}$ can lead to a tighter bound as 
\begin{align}
  \left| {\sum\nolimits_{i = 1}^L {{m_t}({{\hat x}_i}) - \sum\nolimits_{i = 1}^L {f({{\hat x}_i})} } } \right| \le \left\| h \right\|_{{{\cal H}_{{k^\sigma }}}}^{}\sqrt {{{\bf{1}}^T}{\bf{A1}} + L{\sigma ^2}} + L \left\| g \right\|_{{{\cal H}_{{\sigma^2\delta }}}} \sigma
\end{align}

\begin{proof}
Let ${\bf{\alpha }}^i = ({\bf{K}}_t + \sigma^2I) ^ {-1}{{\bf{k}}_t}{({\widehat x_i})}$. Then we have 
\begin{align}
{{{\left( {\sum\limits_{i = 1}^L {{\widehat m_t}({{\widehat x}_i})}  - \sum\limits_{i = 1}^L {h({{\widehat x}_i})} } \right)}^2}}&  = {{\left( {\sum\limits_{i = 1}^L {\sum\limits_{l = 1}^t {{\alpha _l^i}h\left( {{{ x}_l}} \right)}  - \sum\limits_{i = 1}^L {h({{\widehat x}_i})} } } \right)}^2} \\
& = {{\left( {\left\langle {\sum\limits_{i = 1}^L {\sum\limits_{l = 1}^t {\alpha _l^ik^\sigma\left( {{x_l}, \cdot } \right)} }  - \sum\limits_{i = 1}^L {k^\sigma({{\widehat x}_i}, \cdot )} ,h} \right\rangle } \right)}^2}\\
& \le \left\| h \right\|_{{{\cal H}_{k^\sigma}}}^2\left\| {\sum\limits_{i = 1}^L {\sum\limits_{l = 1}^t {\alpha _l^ik^\sigma\left( {{x_l}, \cdot } \right)} }  - \sum\limits_{i = 1}^L {k^\sigma({{\widehat x}_i}, \cdot )} } \right\|_{{{\cal H}_{k^\sigma}}}^2
\end{align}
In addition, we have
\begin{align}
\nonumber
& \left\| {\sum\limits_{i = 1}^L {\sum\limits_{l = 1}^t {\alpha _l^ik^\sigma\left( {{x_l}, \cdot } \right)} }  - \sum\limits_{i = 1}^L {k^\sigma({{\widehat x}_i}, \cdot )} } \right\|_{{{\cal H}_{k^\sigma}}}^2 \\ & = \sum\limits_{i = 1}^L {\sum\limits_{j = 1}^L {k^\sigma({{\widehat x}_i},{{\widehat x}_j})} }  - 2\sum\limits_{i = 1}^L {\sum\limits_{j = 1}^L {\sum\limits_{l = 1}^t {\alpha _l^ik^\sigma\left( {{x_l},{{\widehat x}_j}} \right)} } }  + \sum\limits_{i = 1}^L {\sum\limits_{j = 1}^L {\sum\limits_{n = 1}^t {\sum\limits_{l = 1}^t {\alpha _l^i\alpha _n^jk^\sigma\left( {{x_l},{x_n}} \right)} } } } \\
& \le \sum\limits_{i = 1}^L {\sum\limits_{j = 1}^L {k({{\widehat x}_i},{{\widehat x}_j})} } + L^2\sigma^2  - 2\sum\limits_{i = 1}^L {\sum\limits_{j = 1}^L {{{\bf{k}}_t}{{({{\widehat x}_i})}^T}({\bf{K}}_t + \sigma^2I) ^ {-1} {{\bf{k}}_t}({{\widehat x}_j})} } \nonumber \\ & \;\;\;\;\; + \sum\limits_{i = 1}^L {\sum\limits_{j = 1}^L {{{\bf{k}}_t}{{({{\widehat x}_i})}^T}({\bf{K}}_t + \sigma^2I) ^ {-1} {{\bf{k}}_t}({{\widehat x}_j})} } \\
& = \sum\limits_{i = 1}^L {\sum\limits_{j = 1}^L {{{\bf{A}}_{ij}} + L^2\sigma^2 = {{\bf{1}}^T}{\bf{A1}}} } + L^2\sigma^2
\end{align}

Thus, we obtain ${\left( {\sum\nolimits_{i = 1}^L {{m_t}({{\widehat x_i}}) - \sum\nolimits_{i = 1}^L {h({{\widehat x_i}})} } } \right)^2} \le \left\| h \right\|_{{\mathcal{H}_{k^\sigma}}}^2  ({{\bf{1}}^T}{\bf{A1}}+ L^2\sigma^2)$. Then, we can achieve that
\begin{align}
    \left| {\sum\nolimits_{i = 1}^L {{m_t}({{\hat x}_i}) - \sum\nolimits_{i = 1}^L {f({{\hat x}_i})} } } \right| & \le \left| {\sum\nolimits_{i = 1}^L {{m_t}({{\hat x}_i}) - \sum\nolimits_{i = 1}^L {h({{\hat x}_i})} } } \right| + \sum\nolimits_{i = 1}^L {\left| {g({{\hat x}_i})} \right|}  \\
    & \le \left\| h \right\|_{{{\cal H}_{{k^\sigma }}}}^{}\sqrt {{{\bf{1}}^T}{\bf{A1}} + L^2{\sigma ^2}} + L\left\| g \right\|_{{{\cal H}_{{\sigma^2\delta }}}} \sigma
\end{align}

\end{proof}

\begin{lemma}
Suppose $h= f+g \in \mathcal{H}_k^\sigma$ associated with kernel  $k^\sigma(x,y) = k(x,y)+ \sigma^2\delta(x,y) $. Suppose $f \in \mathcal{H}_k$ associated with $k$ and $g \in \mathcal{H}_{\sigma^2 \delta}$ associated with kernel $\sigma^2 \delta $. Suppose  $ x_i \ne x_j$,   then we have
\begin{align}
  \frac{1}{L}\sum\limits_{i = 1}^L {\left( {f({x^*}) - f({x_{(n - 1)L + i}})} \right)}  \le 2\left\| h \right\|_{{\mathcal{H}_{k^\sigma}}}  \sqrt { \frac{tr\left( {\widehat { \mathop{\rm cov}}_{n - 1}({X_n},{X_n})} \right)}{L}}  + 2 \left( {\left\| h \right\|_{{{\cal H}_{{k^\sigma }}}}^{} + \left\| g \right\|_{{{\cal H}_{{\sigma ^2}\delta }}}^{}} \right){\sigma ^{}}
\end{align}
 where covariance matrix ${\widehat { \mathop{\rm cov}}_{{n - 1}}({X_n},{X_n})}$ is constructed as Eq.(\ref{NoiseBatchV}) with $X_n = \{x_{(n - 1)L + 1},...,x_{nL}\}$.
\end{lemma}

\begin{proof}

Let $X^*=\{x^*,...,x^*\}$ be $L$ copies of $x^*$. Then, we obtain that

\begin{align}
& \frac{1}{L}\sum\limits_{i = 1}^L {\left( {f({x^*}) - f({x_{(n - 1)L + i}})} \right)}  = f({x^*}) - \frac{1}{L}\sum\limits_{i = 1}^L {f({x_{(n - 1)L + i}})} \\
& \le {\widehat m_{(n - 1)L}}({x^*}) + {\left\| h \right\|_{{{\cal H}_{k^\sigma}}}}{ \widehat \sigma _{(n - 1)L}}({x^*}) + \left( {\left\| h \right\|_{{{\cal H}_{{k^\sigma }}}}^{} + \left\| g \right\|_{{{\cal H}_{{\sigma ^2}\delta }}}^{}} \right){\sigma}  - \frac{1}{L}\sum\limits_{i = 1}^L {f({x_{(n - 1)L + i}})} \\ \nonumber
& = \frac{1}{L}\sum\limits_{i = 1}^L {{\widehat m_{(n - 1)L}}({x^*})}  + {\left\| h \right\|_{{{\cal H}_{k^\sigma}}}}\left( {2\sqrt {\frac{{tr\left( {{\widehat { \mathop{\rm cov}}_{{n - 1}}}({X^*},{X^*})} \right)}}{L}}  - \sqrt {\frac{{{{\bf{1}}^T}{\widehat { \mathop{\rm cov}}_{{n - 1}}}({X^*},{X^*}){\bf{1}}}}{{{L^2}}}} } \right) \\  & + \left( {\left\| h \right\|_{{{\cal H}_{{k^\sigma }}}}^{} + \left\| g \right\|_{{{\cal H}_{{\sigma ^2}\delta }}}^{}} \right){\sigma}  - \frac{1}{L}\sum\limits_{i = 1}^L {f({x_{(n - 1)L + i}})} \\  \nonumber
& \le \frac{1}{L}\sum\limits_{i = 1}^L {{\widehat m_{(n - 1)L}}({x_{(n - 1)L + i}})}  + {\left\| h \right\|_{{{\cal H}_{k^\sigma}}}}\left( {2\sqrt {\frac{{tr\left( {{\widehat { \mathop{\rm cov}}_{{n - 1}}}({X_n},{X_n})} \right)}}{L}}   - \sqrt {\frac{{{{\bf{1}}^T}{\widehat { \mathop{\rm cov}}_{{n - 1}}}({X_n},{X_n}){\bf{1}}}}{{{L^2}}}} } \right) \\  &  + \left( {\left\| h \right\|_{{{\cal H}_{{k^\sigma }}}}^{} + \left\| g \right\|_{{{\cal H}_{{\sigma ^2}\delta }}}^{}} \right){\sigma} -  \frac{1}{L}\sum\limits_{i = 1}^L {f({x_{(n - 1)L + i}})} \\ \nonumber
& \le  {\left\| h \right\|_{{{\cal H}_{k^\sigma}}}}\left( {2\sqrt {\frac{{tr\left( {{\widehat { \mathop{\rm cov}}_{{n - 1}}}({X_n},{X_n})} \right)}}{L}}  - \sqrt {\frac{{{{\bf{1}}^T}{\widehat { \mathop{\rm cov}}_{{n - 1}}}({X_n},{X_n}){\bf{1}}}}{{{L^2}}}} } \right) + \left( {\left\| h \right\|_{{{\cal H}_{{k^\sigma }}}}^{} + \left\| g \right\|_{{{\cal H}_{{\sigma ^2}\delta }}}^{}} \right){\sigma} \\ &  +  {\left\| h \right\|_{{{\cal H}_{k^\sigma}}}}\sqrt {\frac{{{{\bf{1}}^T}{\widehat { \mathop{\rm cov}}_{{n - 1}}}({X_n},{X_n}){\bf{1}}}+ L^2 \sigma^2}{{{L^2}}}} +  \left\| g \right\|_{{{\cal H}_{{\sigma^2\delta }}}} \sigma \\
& \le 2 {\left\| h \right\|_{{{\cal H}_{k^\sigma}}}}\sqrt {\frac{{tr\left( {{\widehat { \mathop{\rm cov}}_{{n - 1}}}({X_n},{X_n})} \right)}}{L}}  
+ 2 \left( {\left\| h \right\|_{{{\cal H}_{{k^\sigma }}}}^{} + \left\| g \right\|_{{{\cal H}_{{\sigma ^2}\delta }}}^{}} \right){\sigma}
\end{align}

\end{proof}

Finally, we are ready to attain Theorem 4.
\begin{proof}
Let ${{\bf{A}}_{n - 1}}={\widehat { \mathop{\rm cov}}_{{n - 1}}({X_n},{X_n})}$ be the covariance matrix constructed as Eq.(\ref{NoiseBatchV}) with $X_n = \{x_{(n - 1)L + 1},...,x_{nL}\}$. Let $\beta_{n-1}= \left\| {{\bf{A}}_{n - 1}} \right\|_2$.  Then, we can achieve that
\begin{align}
    {R_T} & = \sum\limits_{t = 1}^T {f({x^*}) - f({x_t})} \\
& \le 2 {\left\| h \right\|_{{{\cal H}_{k^\sigma}}}}\sum\limits_{n = 1}^N \sqrt {L \; tr\left( {{{\bf{A}}_{n-1}}} \right)} + 2 T\left( {\left\| h \right\|_{{{\cal H}_{{k^\sigma }}}}^{} + \left\| g \right\|_{{{\cal H}_{{\sigma ^2}\delta }}}^{}} \right){\sigma}  \\
&  \le 2 {\left\| h \right\|_{{{\cal H}_{k^\sigma}}}}\sqrt {NL\sum\limits_{n = 1}^N {\;tr\left( {{{\bf{A}}_{n - 1}}} \right)} } + 2 T\left( {\left\| h \right\|_{{{\cal H}_{{k^\sigma }}}}^{} + \left\| g \right\|_{{{\cal H}_{{\sigma ^2}\delta }}}^{}} \right){\sigma}  \\
& \le 2 {\left\| h \right\|_{{{\cal H}_{k^\sigma}}}}\sqrt {T\sum\limits_{n = 1}^N {\;\frac{{{\beta _{n - 1}}}}{{\log \left( {1 + {\beta _{n - 1}}{\sigma ^{ - 2}}} \right)}}\log \det \left( {I + {\sigma ^{ - 2}}{{\bf{A}}_{n - 1}}} \right)} } + 2 T\left( {\left\| h \right\|_{{{\cal H}_{{k^\sigma }}}}^{} + \left\| g \right\|_{{{\cal H}_{{\sigma ^2}\delta }}}^{}} \right){\sigma} \\
& \le  {\left\| h \right\|_{{{\cal H}_{k^\sigma}}}}\sqrt {T{C_4}\sum\limits_{n = 1}^N {\;\log \det \left( {I + {\sigma ^{ - 2}}{{\bf{A}}_{n - 1}}} \right)} } + 2 T\left( {\left\| h \right\|_{{{\cal H}_{{k^\sigma }}}}^{} + \left\| g \right\|_{{{\cal H}_{{\sigma ^2}\delta }}}^{}} \right){\sigma} \\
& \le {\left\| h \right\|_{{{\cal H}_{k^\sigma}}}}\sqrt {T{C_4}{\gamma _T}} + 2 T\left( {\left\| h \right\|_{{{\cal H}_{{k^\sigma }}}}^{} + \left\| g \right\|_{{{\cal H}_{{\sigma ^2}\delta }}}^{}} \right){\sigma} 
\end{align}

It follows that $r_T \le \frac{{{R_T}}}{T} \le {\left\| h \right\|_{{{\cal H}_{k^\sigma}}}}\sqrt {\frac{{C_4{\gamma _T}}}{T}}+ 2 T\left( {\left\| h \right\|_{{{\cal H}_{{k^\sigma }}}}^{} + \left\| g \right\|_{{{\cal H}_{{\sigma ^2}\delta }}}^{}} \right){\sigma} $

\end{proof}

\section{Proof of Theorem \ref{Sobolev}}
\begin{proof}

\begin{align}
 {\widetilde{r}_T} & = {\min _{t \in [T] }}\sup_{f_t\in \mathcal{B}_k, f_t(x_i)=f_i(x_i), \forall i \in [t-1]} \{f_t(x^*) - f_t({x_t}) \} \nonumber \\ 
& \le \sup_{f_T\in \mathcal{B}_k, f_T(x_i)=f_i(x_i), \forall i \in [T-1]} \{f_T(x^*) - f_T({x_T}) \} \nonumber \\
& \le \sup_{f_T\in \mathcal{B}_k, f_T(x_i)=f_i(x_i), \forall i \in [T-1]} \{m_{T-1}(x^*) + B\sigma_{T-1}(x^*) - f_T({x_T}) \} \nonumber \\
& \le \sup_{f_T\in \mathcal{B}_k, f_T(x_i)=f_i(x_i), \forall i \in [T-1]} \{m_{T-1}(x_T) + B\sigma_{T-1}(x_T) - f_T({x_T}) \} \nonumber \\
& \le \sup_{f_T\in \mathcal{B}_k, f_T(x_i)=f_i(x_i), \forall i \in [T-1]} \{  B\sigma_{T-1}(x_T) +B\sigma_{T-1}(x_T)\} \nonumber \\
& \le 2B \sigma_{T-1}(x_T) \label{sobolevB1}
\end{align}


Applying Theorem 5.4 in \cite{kanagawa2018gaussian} with $h_{\rho,X}\le h_X$, we can obtain that 
\begin{align}\label{sobolevB2}
   \sigma_{T-1}(x_T) \le C h^{s-d/2}_X 
\end{align}

Together with (\ref{sobolevB1}) and (\ref{sobolevB2}), absorbing the constant into $C$, we can achieve that ${\widetilde{r}_T} \le C h^{s-d/2}_X $

\end{proof}

\section{Proof of Theorem \ref{SE}}
\begin{proof}
From , we know $ {\widetilde{r}_T} \le 2B \sigma_{T-1}(x_T)$. By applying Theorem 11.22 in \cite{wendland2004scattered}, we can obtain that 
\begin{align}\label{SEB2}
   2B\sigma_{T-1}(x_T) \le 2B \exp(c\log(h_X)/(2\sqrt{h_X}))
\end{align}
It follows that  ${\widetilde{r}_T} \le 2B \exp(c\log(h_X)/(2\sqrt{h_X}))$.
\end{proof}

\section{Proof of Theorem \ref{FullyAdv}}

\begin{proof}
\begin{align}
     {\widetilde{r}_T} & = {\min _{t \in [T] }}\sup_{f_t\in \mathcal{B}_k}\{f_t(x^*) - f_t({x_t}) \} \nonumber \\
     & = {\min _{t \in [T] }}\sup_{f_t\in \mathcal{B}_k}\{ {\left\langle { k(x^*, \cdot ) - k(x_t, \cdot ),f_t} \right\rangle }\}\nonumber \\
& \le {\min _{t \in [T] }} B \sqrt{\left\langle  k(x^*, \cdot ) - k(x_t, \cdot ),  k(x^*, \cdot ) - k(x_t, \cdot ) \right\rangle} \nonumber \\
     & \le {\min _{t \in [T] }} B \sqrt{ (k(x^*,x^*)+k(x_t,x_t) - 2 k(x^*,x_t))} \nonumber \\
    & \le B \sqrt{(2-2\Phi(h_X))}
\end{align}
\end{proof}

\section{Proof of Corollary \ref{SEbound}}
\begin{proof}
From Theorem  \ref{FullyAdv}, we can obtain that 
\begin{align}
     {\widetilde{r}_T} & = {\min _{t \in [T] }}\sup_{f_t\in \mathcal{B}_k}\{f_t(x^*) - f_t({x_t}) \} \nonumber \\
    & \le B \sqrt{(2-2\Phi(h_X))} \nonumber \\
    & = B \sqrt{(2-2\exp(-Ch_X^2))}  \nonumber \\
    & \le B\sqrt{2(Ch_X^2)} \nonumber \\
    & = B\sqrt{2C}h_X = \mathcal{O}(h_X)
\end{align}
\end{proof}

\end{document}